\documentclass[smallextended]{svjour3}
\usepackage{hyperref}

\usepackage{amssymb, amsfonts, amsmath}

\usepackage{amsthm}
\usepackage [autostyle, english = american]{csquotes}
\usepackage{tabularx}
\usepackage{graphicx}
\usepackage{url}
\usepackage{bbm}
\usepackage{algorithm}
\usepackage{algorithmic}

\MakeOuterQuote{"}

\renewcommand{\P}{\mathcal{P}}

\newcommand{\X}{\mathcal{X}}

\newcommand{\B}{\mathcal{B}}
\newcommand{\Q}{\mathcal{Q}}

\newcommand{\Z}{\mathcal{Z}}

\newcommand{\size}[1]{\lvert #1 \rvert}

\newcommand{\figref}[1]{Figure~\ref{fig:#1}}
\newcommand{\tabref}[1]{Table~\ref{tab:#1}}
\renewcommand{\eqref}[1]{Equation~\ref{eq:#1}}
\newcommand{\eqsref}[2]{Equations~\ref{eq:#1} and~\ref{eq:#2}}

\newcommand{\secref}[1]{Section~\ref{sec:#1}}
\newcommand{\secstworef}[2]{Sections~\ref{sec:#1} and~\ref{sec:#2}}
\newcommand{\appref}[1]{Appendix~\ref{app:#1}}
\newcommand{\thmref}[1]{Theorem~\ref{thm:#1}}

\newtheorem*{inftheorem}{Theorem}
\newtheorem*{inflemma}{Lemma}

\newcommand{\lemref}[1]{Lemma~\ref{lem:#1}}
\newcommand{\lemstworef}[2]{Lemmas~\ref{lem:#1} and~\ref{lem:#2}}

\newcommand{\defref}[1]{Definition~\ref{def:#1}}

\newcommand{\corref}[1]{Corollary~\ref{cor:#1}}

\newcommand{\algref}[1]{Algorithm~\ref{alg:#1}}
\newcommand{\lineref}[1]{Line~\ref{lst:line:#1}}
\newcommand{\linestworef}[2]{Lines~\ref{lst:line:#1} and~\ref{lst:line:#2}}
\newcommand{\linestwosref}[2]{Lines~\ref{lst:line:#1} to~\ref{lst:line:#2}}

\newcommand{\mmin}[1]{\underset{#1}{\text{min }}}
\newcommand{\mmax}[1]{\underset{#1}{\text{max }}}

\renewcommand{\sup}[1]{\underset{#1}{\text{sup }}}

\newcommand{\BigO}[1]{\ensuremath{\operatorname{O}\!\left(#1\right)}}

\newcommand{\BigOmega}[1]{\ensuremath{\operatorname{\Omega}\!\left(#1\right)}}
\newcommand{\BigTheta}[1]{\ensuremath{\operatorname{\Theta}\!\left(#1\right)}}

\newcommand{\train}{\textsc{train-model}}
\DeclareMathOperator{\vol}{vol}
\def\orcidID#1{\unskip$^{[#1]}$}

\smartqed
\newcommand\newcite{\cite}

\begin{document}

\title{An Algorithm for Learning Representations of Models With Scarce Data}

\author{Adrian de Wynter}
\authorrunning{A. de Wynter}

\institute{$^1$Microsoft Corporation. Work done while at Amazon Alexa. \\
  \email{adewynter@microsoft.com} \\
  \orcidID{ORCID: 0000-0003-2679-7241}
}

\date{Received: date / Accepted: date}

\maketitle

\begin{abstract}%

We present an algorithm for solving binary classification problems when the dataset is not fully representative of the problem being solved, and obtaining more data is not possible. 
It relies on a trained model with loose accuracy constraints, an iterative hyperparameter searching-and-pruning procedure over a search space $\Theta$, and a data-generating function. 
Our algorithm works by reconstructing up to homology the manifold on which lies the support of the underlying distribution. 
We provide an analysis on correctness and runtime complexity under ideal conditions and an extension to deep neural networks. 
In the former case, if $\size{\Theta}$ is the number of hyperparameter sets in the search space, 
this algorithm returns a solution that is up to $2(1 - {2^{-\size{\Theta}}})$ times better than simply training with an enumeration of $\Theta$ and picking the best model. 
As part of our analysis we also prove that an open cover of a dataset has the same homology as the manifold on which lies the support of the underlying probability distribution, if and only said dataset is learnable. 
This latter result acts as a formal argument to explain the effectiveness of data expansion techniques. 
\end{abstract}

\section{Introduction}

The goal of most machine learning systems is to, given a sample of an unknown probability distribution, fit a model that generalizes well to other samples drawn from the same distribution. 
It is often assumed that said sample is fully representative of the unknown, or \emph{underlying}, probability distribution being modeled, as well as being independent and identically distributed (i.i.d.) \cite{ShalevUML}. 

It occurs in many natural problems that at least one part of said assumption does not hold. This could be by either noisy or insufficient sampling of the underlying function, or an ill-posed problem, where the features of the data do not represent the relevant characteristics of the phenomenon being described \cite{DundarEtAl,Hampel}. 
This issue, however, becomes more self-evident when training small, shallow models, where exposure to large amounts of representative data is critical for proper generalizability \cite{Hanneke,kawaguchi2017generalization,Livni}.

Moreover, in recent years, the emergence of large language models (LLMs), and their successful application as data generators for smaller models \cite{NEURIPS2023_ae9500c4,YangEtAl,phi1} poses the question as to whether there is a formal way to understand this effectiveness.

\subsection{Contributions}
We present an algorithm, which we dub Agora,\footnote{One of Plato's Dialogues---\emph{Timaeus}---involves a conversation between the eponymous character and Socrates, about the nature of the world and how it came into being. It fits well with the nature of our algorithm, and with the goal of machine learning.} for solving binary classification problems when the data does not fully characterize the underlying probability distribution, and further data acquisition is not possible. 
Note that this further data acquisition does not include synthetic data expansion techniques. 

Our algorithm works by combining hyperparameter optimization and data augmentation. At every step it:

\begin{enumerate}
\item Trains a model (\emph{Timaeus})
\item Evaluates Timaeus, and collects the data points that were mislabeled from the evaluation set.
\item It then generates a new dataset on these mislabeled data points with a data-generating function (the $\tau$-\emph{function}). This function is an abstract construct, but could be any generative strategy, such as an LLM, a grammar, divergence-based sampling between manifolds. 
\item It labels this new dataset with a trained, but relatively low-performing, model (\emph{Socrates}), and concatenates it to the training set.
\end{enumerate}
Agora adapts to the evolving dataset by searching and pruning multiple hyperparameter sets $\theta \subset \mathbb{R}$ from a \emph{search space} $\Theta$. 
All of the algorithm-specific nomenclature can be found in \tabref{nomenclature}, and the full algorithm in \algref{merlinalg}. 

\subsubsection{Correctness Bounds for Agora}

We provide correctness (\thmref{maintheorem}) and runtime (\thmref{timecomplexity}) bounds for our algorithm. 
Barring some technical assumptions, we prove that Agora iteratively builds a representative dataset that can be learned with Timaeus, by reconstructing (up to homology) the manifold on which lies the support of the underlying distribution being learned. 
A diagram of this is in \figref{algorithmdiagram}. 

Compared to simply picking the best trained model over all $\theta \in \Theta$, we prove that Agora's solution is up to $2\left(1 - \frac{1}{2^{\size{\Theta}}}\right)$ times better; where $\size{\Theta}$ is the cardinality (number of candidate hyperparameter sets) of the search space. 
This result holds if the underlying probability distribution $\P$ is supported on a compact Riemannian manifold, and Timaeus is unable to "forget" a correctly-labeled point. 
We also provide experimental results in \appref{experiments}.

\subsubsection{Effectiveness of Data-Expansion Techniques}

We prove that under some technical assumptions, an open cover of a dataset $D$ sampled from $\P$ has the same homology as the manifold where $\P$ is supported on, if and only if $D$ is learnable.%
While we only leverage this result as a mechanism for our correctness proofs, we believe that this result is of interest to the community, as it brings in techniques from learning theory and topological data analysis and serves as a formal proof of when and how data expansion techniques work over statistical manifolds. 

\begin{table}
\centering
\begin{tabular}{ |c|c| } \hline
 Component name & Purpose in Agora \\ \hline\hline
 Timaeus                     & Model to be trained \\
 Socrates                    & Pre-trained model for labeling new points \\
 $\tau$-function             & Noisy data-generating function \\ 
 Hyperparameter set $\theta$ & Single set of hyperparameters for the training algorithm, $\theta \subset \mathbb{R}$ \\
 Search space $\Theta$       & Set of distinct hyperparameter sets $\theta$ \\ \hline
\end{tabular}
\caption{Definitions for the key components of Agora.}
\label{tab:nomenclature}
\end{table}

\begin{figure}
    \centering
    \includegraphics[width=\columnwidth]{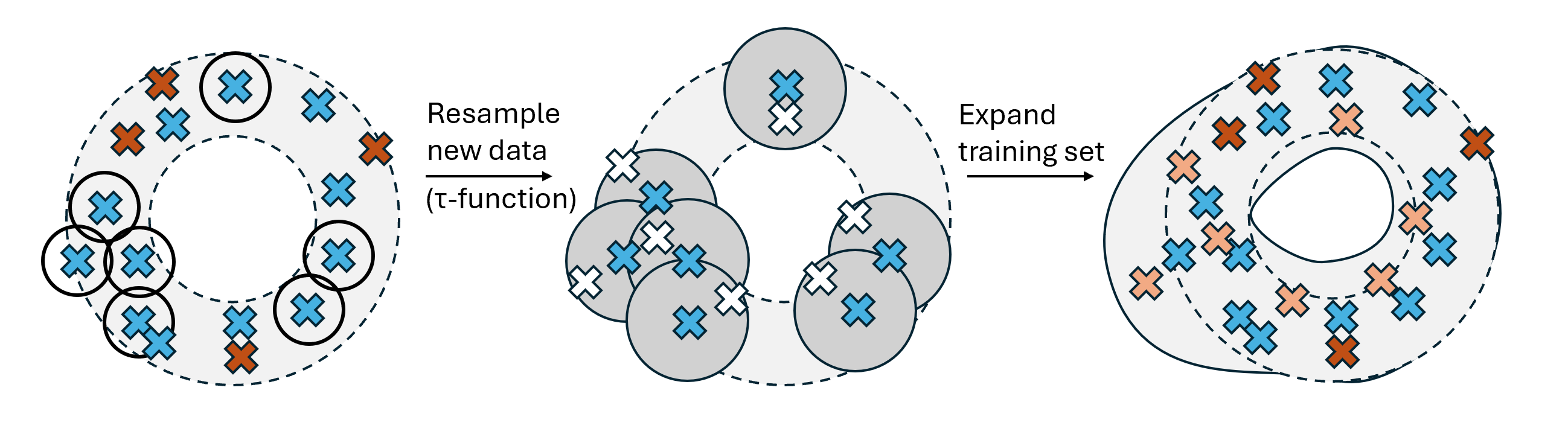}
    \caption{Intuition behind our algorithm. The training set ($D$) is denoted by red crosses, and the evaluation set ($E$) by blue crosses. 
    The instance space (phenomenon being modeled) is an annulus, but $D$ is not sufficient to describe (learn) it. 
    \textit{Left}: $D$ and $E$ superimposed on the instance space. 
    Circled in black are elements from $E$ set that the model mislabeled. 
    \textit{Middle}: The $\tau$-function creates new data within a neighborhood of the mislabeled points (white crosses). 
    \textit{Right}: The new points are labeled with Socrates and added to $D$ (lighter red crosses), leaving $E$ intact. Observe how the new instance space is equivalent (up to homology) to the original instance space. 
    We show that, barring some technical conditions, the $\tau$-function creates a dataset whose open cover has the same homology than the original instance space, and it is learnable. 
    This is allows models to learn the phenomenon when $D$ is not representative, as in the picture. 
    }
    \label{fig:algorithmdiagram}
\end{figure}

\subsection{Outline}
The rest of this paper is structured as follows: 
we provide a brief survey of related approaches to this and similar problems in \secref{relatedwork}. 
After establishing notation, basic definitions, and assumptions in \secref{background}, we introduce Agora in \secref{algorithm}. 
We provide an analysis of its correctness and asymptotic bounds in \secstworef{proofs}{timebounds}, before concluding in \secref{conclusion} with a discussion of our work, as well as potential applications---and pitfalls---when using this algorithm.

\section{Related Work}\label{sec:relatedwork}

The algorithm and proofs we present in this paper draw from multiple machine learning subfields. 

\subsection{Knowledge Distillation}
Using a pre-trained model to improve the performance of another model is known as knowledge distillation. It was originally proposed by Bucilu\v{a} et al. \newcite{bucilu2006model} as a model compression technique, and later applied to deep learning with much success \newcite{sslref}. 
Since Agora labels new points automatically, it can be viewed as a type of active learning \newcite{balcanactive,HannekeALTheory,BurrActive} related to membership query synthesis \cite{Angluin}, and analogous to the work of Balasubramanian et al. \newcite{Balasubramanian} and Sener and Savarese \newcite{sener2018active}.  
However, active learning focuses on selecting a subset of the data to be labeled, for a specific input task and model. Agora is designed to be task-agnostic, and in our setting we are unabled to sample more data. Moreover, our labeler does not need to have a good accuracy on the dataset: our proofs only require a lower-bound of $2/3$, and rely instead on inducing provably good perturbations in the dataset. 

\subsection{Data Expansion}
Perturbing the dataset is a common approach when applied to deep learning \cite{lin2013network}, and it is known that this technique is able to improve generalizability \cite{Graham}. 
When facing a task with scarce data, another commonly employed technique is the use of generative models to create synthetic data \cite{ho2019pba,alex2017learning,zhang2019adversarial}. 
This is not always a reliable strategy: 
theoretical research around the effectiveness of these approaches can be found in Dao et al. \newcite{KernelData} and Wu et al. \newcite{wu2020generalization}, but focused on specific datasets and approaches, while ours is agnostic to the phenomenon. 
See Oliver et al. \newcite{OliverEtAL} for an overview of both data augmentation and knowledge distillation, and how they fit on the wider field of semi-supervised learning. 
These techniques often focus on the training set without consideration as to whether they belong to the support of the underlying probability distribution, %
which implies that most algorithms for synthetic data generation are task-dependent, unlike our work. 
Recent techniques for model training have relied on a LLM to generate task-related data. Although it is not without its drawbacks (e.g., perpetuating biases \cite{NEURIPS2023_ae9500c4,YangEtAl}), it has proven successful for the creation of both specialized and generalized models, such as Phi-1 \cite{phi1} or Alpaca \cite{alpaca}. 
See Yu et al. \cite{NEURIPS2023_ae9500c4} and Yang et al. \cite{YangEtAl} for current surveys of this subject. 

\subsection{Learnability and Homology}

Considering scenarios where some parts of the i.i.d. assumption do not hold has become a more active field of research, as the emerging field of decentralized learning gains momentum. See, for example the applied works by Dundar et al. \newcite{DundarEtAl}, and Hsieh et al. \newcite{hsieh2020the} from a deep learning point of view. 
There is a trove of work done in learning theory to determine when is a hypothesis class able to learn and generalize for an input task \cite{balcanactive,Hanneke,KearnsEncryption,Valiant}, and more specialized to deep neural networks \cite{kawaguchi2017generalization}, and online learning \cite{Littlestone,ShaiOnline}. 
It is known that most learning algorithms allow some form of error-tolerant statistical query \cite{KearnsSQ}. 
We leverage these results for our proofs, in addition to techniques from topological data analysis. 
The interested reader is enouraged to review the survey by Wasserman \newcite{Wasserman}, and the brief but thorough note by Weinberger \newcite{Weinberger}, for introductions to this field. 
We must highlight that our proof of the correspondence between homology and learnability is, to the best of our knowledge, new. That being said, it relies on the well-known result by Niyogi et al. \newcite{NiyogiSmaleWeinberger}, who stated the minimum sample size needed to reconstruct a manifold. Further positive results around the reconstructability of a manifold under noisy conditions can be found in Genovese et al. \newcite{GenoveseEtAl} and Chazal et al. \newcite{ChazalRobust}. 
Similar arguments using the packing number are given by Mohri et al. \newcite{MohriEtAl} and a number of related works \cite{AlonEtAl,HausslerPacking}; however, said result is used to bound various complexity measures, rather than the data itself. Crucially, they do not leverage the homology of the manifold where the probability distribution is supported on, which, as we point out, it is easier to work with than other similarity measures. 

\subsection{Boosting and Hyperparameter Optimization}

A well-known techinque to improve the performance of a set of weak learners is boosting \cite{KearnsAndValiant,Schapire}, which is a class of meta-algorithms where the predictions of said learners are aggregated to obtain the final result. 
There are modifications of boosting algorithms for scarce-data scenarios, which rely on a prior. 
For example, a variation of LogitBoost \cite{LogitBoost} uses a human-in-the-loop component (e.g., an annotator) as the prior to to enforce learnability conditions \cite{SchapireAndFreund}. Others, like ASSEMBLE \cite{Assemble}, make instead use of pseudoclasses (the sign of the ensemble on a point) and regularization strategies to achieve generalizability. 
Agora, which, from this perspective, is also a meta-algorithm, is focused on a single learner. It generates the data automatically, based on the relationship between the hypothesis class and the number of samples needed for manifold reconstructability. 
It also performs a greedy search over the hyperparameter space, and, for simplicity, we do not focus on more complex algorithms for hyperparameter optimization (HPO). 
Other heuristics, such as grid search, random search \cite{10.5555/2188385.2188395}, and Bayesian optimization methods \cite{10.5555/2999325.2999464} are quite popular. 
HPO is a very active area of research given its associated complexities: for example, optimization must be done over variables not commonly seen in single-model training (e.g. mixtures of integers and real numbers), and require intensive resources to run (which is of concern with models like deep neural networks) \cite{luohpo,YANG2020295,AutomatedHPOML}. 
See the survey by Yang and Shami \cite{YANG2020295} for a deeper introduction to the subject.

\section{Background: Notation and Definitions}\label{sec:background}

\subsection{Notation}\label{sec:motivation}

We begin by introducing the notation and terminology used across this paper. 
Let $X \subset \mathbb{R}^m$ be a nonempty set equipped with a $\sigma$-algebra, called the \emph{instance space}. Also let $Y = \{0, 1\}$ be a set called the \emph{label space}. 
Suppose we have a nonempty sample $D_X = \{x_1, \dots, x_n\}$ drawn with a unknown probability distribution $\P$ defined over $X$, and labeled with a fixed-but-unknown function, or \emph{concept}, $c \colon X \rightarrow Y$. 

Following Valiant \cite{Valiant}, the goal of the binary classification problem is to, given a \emph{dataset} $D = \{\langle x_i, y_i \rangle \colon x_i \in D_X, y_i = c(x_i) \}_{i \in [1, n]}$, output a measurable function (\emph{model}) $f \colon X \rightarrow Y$ from a nonempty \emph{hypothesis class} $W$ via a polynomial-time algorithm (a map) $\train \colon W \times X \times Y \rightarrow W$ such that the true error, 
\begin{equation}\label{eq:errorprobfunction}
\text{Pr}_{ x \sim \P}[f(x) \neq c(x) ] \leq \epsilon,
\end{equation}

\noindent is satisfied with probability at least $1 - \delta$ for some $0 < \epsilon, \delta \leq 1/2$. 

It is common to approximate \eqref{errorprobfunction} by considering instead either the empirical error rate  
\begin{equation}\label{eq:errorfunction}
\text{err}(f(\cdot), E) = \frac{1}{\size{E}}\sum_{\langle x_i, y_i \rangle \in E} \mathbbm{1}[f(x_i) \neq y_i ] ,
\end{equation}

\noindent or the accuracy

\begin{equation}\label{eq:accuracyfunction}
\text{acc}(f(\cdot), E) = 1 - \text{err}(f(\cdot), E)
\end{equation} 

\noindent on a separate \emph{evaluation} dataset $E = \{\langle x_i, y_i \rangle \colon x_i \sim \P, y_i = c(x_i) \}_{i \in [1, p]}$, $E \cap D = \emptyset$, as the main predictor of generalizability of the model. 
Solving this problem directly, however, is known to be hard for multiple problems.\footnote{For example, agnostically learning intersections of half-spaces \cite{Daniely,Klivans}.}

From \eqsref{errorprobfunction}{errorfunction} it is possible to derive a quantity that characterizes a dataset's ability to \emph{represent} the underlying probability distribution $\P$, given some $\delta$, $\epsilon$, and $W$:

\begin{definition}[$(\epsilon, \delta, W)$-representativeness of $\P$]\label{def:erepresentativeness}
Let $D = \{\langle x, y \rangle \colon x \sim \P, y = c(x)\}_{i \in [1, n]}$ be a dataset sampled with an unknown probability distribution $\P$ defined over an instance space $X \subset \mathbb{R}^m$, and labeled with an unknown concept $c \colon X \rightarrow \{0, 1\}$. Let $W$ be a hypothesis class, and $0 < \delta \leq 1/2$. 
We say $D$ is $(\epsilon, \delta, W)$-\emph{representative} (or $(\epsilon, \delta, W)$-learnable) \emph{of} $\P$ if and only if the relation

\begin{equation}\label{eq:vcsamplesize}
\sup{f(\cdot) \in W}\left\vert\text{Pr}_{x\sim \P}[f(x) \neq c(x)] - \frac{1}{\size{D}} \left(\sum_{\langle x_i, y_i \rangle \in D} \mathbbm{1}[f(x_i) \neq y_i ]\right) \right\vert \leq \epsilon
\end{equation}

\noindent holds for a given $0 < \epsilon \leq 1/2$, with probability at least $1-\delta$.
\end{definition}

\defref{erepresentativeness} captures a dataset $D$ and hypothesis class $W$'s ability to capture $\P$. It upper-bounds the difference between the true error and the empirical error rate as incurred by the best model $f(\cdot) \in W$ when trained on $D$. 
It is also referred to as $\epsilon$-\emph{representativeness} \newcite{ShalevUML}. 
Following a result by Hanneke \newcite{Hanneke}, an equivalent definition can be obtained with the minimum sample size for binary function classes of finite VC dimension \cite{VC}.\footnote{The result by Hanneke \newcite{Hanneke} concerns the minimum size of $D$ such that \eqref{errorfunction} holds. It follows that if the sample size doesn't meet that bound, the probability distribution is not $(\epsilon, \delta, W)$-learnable. We revisit this statement in \secref{proofs}.} 

For notational convenience, throughout this paper we adopt a narrower version of this problem and consider our hypothesis space to be a finite set of valid \emph{weight} assignments $w \in W$, $W = \{w_1, \dots, w_k\}$, to a fixed measurable function of the form $f \colon X \times W \rightarrow Y$, and assume that the support of $\P$ is the instance space $X$. 

For example, one such fixed $f(\cdot)$ could be a two-layer neural network accepting inputs $x \in X$, with constant biases $b_1, b_a$, and constrained to have trainable weights from a finite set, say only from $\{\langle w_1, w_a\rangle, \langle w_2, w_b\rangle, \langle w_3, w_c\rangle\}$. 
Hence the hypothesis space $W$ is precisely that set of weights, and one such weight assignment could be:
\begin{equation}
    f(x, \langle w_1, w_a\rangle) = \text{Clamp}\left(w_1 \cdot \text{ReLU}\left(f( w_a \cdot x + b_a)\right) + b_1\right), 
\end{equation}
\noindent where $\text{ReLU}(x) = \max{(0, x)}$ and $\text{Clamp}(x) = \min{(\text{ReLU}(x), 1)}$. 

We also expand the definition of a training algorithm to include an extra set of \emph{hyperparameters} $\theta \subset \mathbb{R}$, so that $\train \colon W \times X \times Y \times \theta \rightarrow W$. 
This set is not learned by any $f(\cdot; w)$, but influences the error rate that can be attained by this model on $E$. 
Remark that these conventions are only for practical purposes, and do not affect the mathematics governing the problem. 

Finally, we use $\B_\rho(x)$ to denote to a $m$-ball of radius $\rho$ and centered at some $x \in \mathbb{R}^m$; and \emph{reach} $\mu$ to refer to the largest number such that any point at a distance $z$ from a compact Riemannian manifold $\X \subset \mathbb{R}^m$, has a unique nearest point in $\X$. 
This quantity, as introduced by Federer \newcite{Federer}, can be interpreted as the amount of curvature of $\X$, and some authors also call it the condition number. 
Throughout this paper we assume that all sets involved are measurable, and $\P$ is supported on some compact Riemannian submanifold of a Euclidean space.\footnote{Real data tends to be high-dimensional, sparse, and not necessarily lying on a manifold. See Niyogi et al. \newcite{NiyogiSmaleWeinberger} and \secref{conclusion} for a discussion on this assumption. } 
This last assumption allows us to introduce the $\tau$-function in \secref{taufunction}, which is critical in our development of Agora. Prior to that, however, we state our problem and its constraints.

\subsection{Problem Statement}\label{sec:statement}

Suppose we have two datasets, $D = \{\langle x_i, y_i \rangle \colon x_i \sim \P, y_i = c(x_i) \}_{i \in [1, n]}$ and $E = \{\langle x_i, y_i \rangle \colon x_i \sim \P, y_i = c(x_i) \}_{i \in [1, p]}$, $E \cap D = \emptyset$, sampled from $\P$ and labeled with an unknown, but fixed, concept class $c$. 
Also suppose $E$ is $(\epsilon, \delta, W)$-representative of $\P$ for a given $\epsilon, \delta$ and a finite, nonempty $W$. 

In this paper we are concerned with the case when $D$ is not $(\epsilon, \delta, W)$-representative of $\P$ for said $\epsilon$, $\delta$, and $W$. We are tasked with returning a hypothesis from $W$ that satisfies \eqref{errorprobfunction}, but we are unable to sample and label more data. Moreover, it is not possible to switch $D$ and $E$ due to---for example---a changing evaluation dataset whose only invariant over time is that it remains $(\epsilon, \delta, W)$-representative of $\P$.

Remark that the scenario where neither $D$ nor $E$ are $(\epsilon, \delta, W)$-representative of $\P$ is readily seen to be impossible to solve\footnote{There is insufficient data to provide a meaningful answer in polynomial time.} and we do not consider it further. 

Our proposed algorithm to solve this problem relies on a mechanism for data augmentation known as the $\tau$\emph{-function}.

\subsection{The $\tau$-function}\label{sec:taufunction}

The $\tau$-function, as introduced in \defref{taufunction} below, is Agora's main device for data augmentation, and it can be seen as a noisy sampling strategy. 
Before formally stating its definition, we provide some intuition for it. 

Since we assume throughout this paper that the support of $\P$, $X$, lies on a compact Riemannian manifold $\X \subset \mathbb{R}^m$, we are able to exploit its structure to generate new points. Suppose we are given a point $x \in X$. The $\tau$-function builds an $m$-ball of radius $\rho/4$ around $x$, and returns a new point $\tilde{x}$ sampled uniformly at random from the $m$-ball. 

The $\tau$-function is noisy since we do not construct the ball on $\X$, but on a separate manifold $\Z \subset \mathbb{R}^m$ such that $\X \subset \Z$. Under some technical assumptions regarding the radius $\rho$, the reach of $\X$ and $\Z$, and the overlap between both manifolds, we prove in \secref{proofs} that, if the centers of the $m$-balls belong to a dataset $(\epsilon, \delta, W)$-representative of $\P$, this approach rebuilds $\X$ up to homology. 
Agora uses this output, along with the labeler, to construct a dataset that is $(\epsilon, \delta, W)$-representative of $\P$. 

\begin{definition}[The $\tau$-function]\label{def:taufunction}
Let $D_X = \{x_1, \dots, x_n \}$ be a set sampled from a probability distribution $\P$ supported on a compact Riemannian manifold $\X \subset \mathbb{R}^m$ with reach $\mu > 0$. 

Let $\Z$ be a compact Riemannian manifold such that $\Z \subset \mathbb{R}^m$ and $\X \subset \Z$. 
Let $\Q_{x, \rho}$ be a uniform probability distribution supported on a set $Z_{x, \rho} = \{\tilde{x} \colon \tilde{x} \in \B_{\rho/4}(x) \land \tilde{x} \in \Z \land \tilde{x}\neq x\}$, for some $\rho > 0$. 

Let $Z = \bigcup_{x \in D_X} Z_{x, \rho}$. 
The $\tau$-function, parametrized by $\rho$, takes in elements $x \in D_X$, and returns a new element drawn from $\Q_{x, \rho}$,

\begin{equation}
\tau^{Z}_\rho \colon D_X \rightarrow Z.
\end{equation}

\end{definition}

While \defref{taufunction} is a topological object, from the point of view of machine learning it is a sampling function: it returns elements from a manifold sufficiently close to our original problem statement, and we use these (noisy) elements to be able to reconstruct it. 
As mentioned earlier, it is designed to abstract out the synthetic data generation aspect of training a model (i.e., it is typically a LLM or a grammar), and prove formal bounds about its performance.

\section{The Agora Algorithm}\label{sec:algorithm}

Agora alternates between data augmentation and hyperparameter pruning. It is displayed in \algref{merlinalg}. 
Throughout this section and the rest of the paper, we use $\Theta = \{\theta_1, \dots, \theta_n\}$ to denote a (finite) set of hyperparameter sets for a choice of training algorithm, and $D$ and $E$ for the training and evaluation datasets. 

At each iteration, Agora selects the highest-performing Timaeus with respect to the evaluation set $E$ and the (available) sets of hyperparameters in the search space $\Theta$. It then augments the training set based on the mistakes made by this high-performing model on $E$. The augmentation is carried out by invoking the $\tau$-function and labeling the new points with Socrates. Finally, it removes the lowest-performing hyperparameters from $\Theta$. 

Remark that the augmentation (and hence the training of Timaeus) is \emph{not} done by training on the mistakes for the evaluation set, but instead by sampling similar points with the $\tau$-function. 

\begin{algorithm}[h]
\caption{The Agora algorithm}\label{alg:merlinalg}
\begin{algorithmic}[1]
   \STATE {\bfseries Input:} Timaeus $f$, Socrates $S$, training dataset $D$, evaluation dataset $E$, hyperparameter sets $\Theta$, $\tau$-function $\tau^{Z}_\rho(\cdot)$

   \STATE $a^{*} \gets - \infty$
   \STATE $w^* \gets \emptyset$

   \WHILE{$\vert \Theta \vert \geq 1$}

   \STATE $Q \gets \{\}$
    \item[]\item[]
   \COMMENT{\textit{Step $1$: Select the highest-performing Timaeus.}}
   \FOR{$\theta_i \in \Theta$} \label{lst:line:startfor}
     \STATE $w^*_i \gets$ \train$(f, D, \theta_i)$ \label{lst:line:trainstatement}
     \STATE $a_i \gets $ acc$(f(\cdot; w^*_i), E)$ %
     \STATE $Q \gets Q \cup \{\langle a_i, \theta_i \rangle \}$

     \IF{$a_i > a^{*}$} \label{lst:line:maxstatement}
       \STATE $w^* \gets w^*_i$
       \STATE $a^{*} \gets a_i$
     \ENDIF
   \ENDFOR \label{lst:line:endfor}
   \item[]\item[]
   \COMMENT{\textit{Step $2$: Augment the dataset.}}
   \STATE $M \gets \{\tau^{Z}_\rho(x) : \langle x ,y \rangle \in E \land f(x; w^{*}) \neq y \land \tau^{Z}_\rho(x) \not\in E \}$ \label{lst:line:sortstatement}
   \STATE $D \gets D \cup \{\langle \tilde{x}, S(\tilde{x}; w^S) \rangle \colon \tilde{x}\in M \}$ \label{lst:line:checkstatement}
    \item[]\item[]
   \COMMENT{\textit{Step $3$: Prune the lowest-performing hyperparameters.}}
   \STATE Sort $Q$ in non-decreasing order with respect to $a_i$. \label{lst:line:sortqstatement}
   \STATE $\theta_{0} \gets$ Select the first hyperparameter set of $Q$.
   \STATE $\theta_{0,j} \gets$ First element of the longest uninterrupted sequence of contiguous hyperparameters in $Q$, starting in $\theta_{0}$, and breaking ties consistently. \label{lst:line:selectstatement}
    \item[]
   \STATE $\Theta \gets \{\theta_t \colon \theta_{0,j} \not\in \theta_t,\; \forall \theta_t \in \Theta\}$ \label{lst:line:removestatement}

   \ENDWHILE
   \RETURN $w^*$

\end{algorithmic}
\end{algorithm}

Formally, Agora takes as an input a tuple $\langle f, S, D, E, \Theta, \tau^{Z}_\rho \rangle$, where:
\begin{itemize}
    \item $f(\cdot; \cdot)$ is an untrained, fixed model (Timaeus) $f\colon X \times W \rightarrow Y$,
    \item $S(\cdot; w^S)$ is a trained model (Socrates) $S\colon X \times \{w^S\} \rightarrow Y$,
    \item $D$ and $E$ are datasets, such that $D = \{\langle x_i, y_i \rangle \colon x_i \sim\P, y_i = c(x_i) \}_{i \in [1, n]}$, $E = \{\langle x_i, y_i \rangle \colon x_i \sim\P, y_i = c(x_i) \}_{i \in [1, p]}$. Also, $D \cap E = \emptyset$; $E$ is drawn i.i.d. from $\P$; and $E$ is $(\epsilon, \delta, W)$-representative of $\P$ for $W$ and some, perhaps unknown, $\epsilon$ and $\delta$. 
    \item $\Theta = \{\theta_1, \dots, \theta_t \}$ is the search space, comprised of hyperparameter sets $\theta_i = \{\theta_{i,1}, \dots \theta_{i, u} \colon \theta_{i, j} \in \mathbb{R}\}$, and
    \item $\tau^{Z}_\rho(\cdot)$ is a $\tau$-function parametrized by some $\rho$ and $Z = \bigcup_{\langle x, y \rangle \in E} Z_{x, \rho}$.
\end{itemize}
Agora also has access to a training subroutine \train$(\cdot, \cdot, \cdot)$ as described in \secref{motivation}. 

At the $k^{\text{th}}$ iteration, the algorithm maintains a list of the current (yet to be pruned) sets of hyperparameter sets, $\Theta^{(k)} \subset \Theta$. While $\Theta$ is nonempty, Agora executes three steps: select Timaeus, augment the dataset, and prune $\Theta^{(k)}$. We describe each of these steps below: 

\begin{enumerate}
    \item \emph{Select Timaeus} (\linestwosref{startfor}{endfor} in \algref{merlinalg}):
    \begin{itemize}
        \item Train $\size{\Theta^{(k)}}$ Timaeus models $f(\cdot; w_i^{*,(k)})$ on $D^{(k)}$ corresponding to every $\theta_i \in \Theta^{(k)}$, and select the highest-accuracy model, $f(\cdot; w^{*,(k)})$, as evaluated on $E$.
        \item Store all pairs of accuracies and corresponding $\theta_i$ for the current iteration in a list $Q^{(k)}$.
    \end{itemize}
    \item \emph{Augment the dataset} (\linestwosref{sortstatement}{checkstatement} in \algref{merlinalg}):
    \begin{itemize}
        \item Create a set of all the mistakes done by the highest-accuracy model, $M^{(k)}$.
        \item Call the $\tau$-function on every member of $M^{(k)}$. 
        \item Label every new point $\tilde{x} \in M^{(k)}$ with Socrates. 
        \item Append the labeled $M^{(k)}$ set to the training set.
    \end{itemize}
    \item \emph{Prune the hyperparameter set} greedily (\linestwosref{sortqstatement}{removestatement} in \algref{merlinalg}):
    \begin{itemize}
        \item Sort $Q^{(k)}$ in non-decreasing order, with respect to the accuracy values.
        \item Let $\theta_0 \in \Theta^{(k)}$ be the first element in $Q^{(k)}$. Select the hyperparameter $\theta_{0,j} \in \theta_0$ belonging to the longest uninterrupted sequence of hyperparameter sets. Namely, $\theta_{0, j}$ is selected if it belongs to the longest uninterrupted sequence $\theta_{0, j} \in \theta_1, \theta_{0, j} \in \theta_2 \dots$, as described by $Q^{(k})$, and when compared to all other $\theta_{0, i} \in \theta_{0}$. 
        \item Remove all the hyperparameter sets from $\Theta^{(k)}$ that contain $\theta_{0,j}$, that is, $\Theta^{(k+1)} = \{\theta_t \colon \theta_{0,j} \not\in \theta_t,\; \forall \theta_t \in \Theta^{(k)}\}$.
    \end{itemize}
\end{enumerate}

Under the correctness conditions stated in \secref{proofs}, Agora iteratively constructs a dataset that is $(\epsilon, \delta, W)$-learnable, by reconstructing the manifold (up to homology) on which lies the support of $\P$.

\begin{remark}
Training Timaeus, especially when it is a deep neural network, may be computationally costly, and dominate the other operations in Agora as the training dataset grows at each iteration. 
Precise time complexity bounds and their comparison to a simple search over all of $\Theta$ can be found in \secref{timebounds}. 
Experiments on its accuracy and runtime when compared to an enumeration over $\Theta$ can be found in \appref{experiments}. 
\end{remark}

\section{Correctness}\label{sec:proofs}

In this section we prove \thmref{maintheorem}, our main result, which states the correctness of Agora. 
It shows that the accuracy of a model trained under the final, expanded dataset obtained from Agora upper bounds all other models. Informally, and assuming as always that $D$ and $E$ are the training and evaluation datasets, and $\Theta$ a set of hyperparameter sets for Agora: 

\begin{inftheorem}[\textbf{\ref{thm:maintheorem}, Informal}]
Let $\text{acc}(f(\cdot; w^*), E)$ be the accuracy of a model on $E$, where $w^* \in W$ is the result of running Agora with this input. 

Let $\text{acc}(f(\cdot; w^e), E) = \mmax{\theta \in \Theta}\{acc(f(\cdot; w), E) \colon w = \train(f, D, \theta)\}$ be the accuracy of a model obtained from an enumeration over all $\Theta$. 
Then, if $D$ and $E$ have equal proportion of positive and negative labels, the following holds:

\begin{equation}\label{eq:ineqbounds1}
\text{acc}(f(\cdot; w^{*}), E) = r \cdot \text{acc}(f(\cdot; w^e), E),
\end{equation}
\noindent for $1 \leq r \leq 2\left(1 - \frac{1}{2^{\size{\Theta}}}\right)$. 

Let $D^*$ and $\theta^*$ be the final dataset and hyperparameter set from running Agora with this instance. 
Then the accuracy of the returned Timaeus model $f(\cdot; w^*)$ on $E$ upper bounds all the other hyperparameter sets $\theta_i \in \Theta$, when trained with the final dataset $D^*$, 

\begin{equation}\label{eq:hpobounds1}
\text{acc}(f(\cdot; w^*), E) \geq \mmax{\theta \in \Theta \setminus \{\theta^*\}}\{\text{acc}(f(\cdot; w), E) \colon w = \train(f, D^*, \theta)\}. 
\end{equation}
\end{inftheorem}

\subsection{Outline of the Proof}

To prove \thmref{maintheorem}, we must first show that a corpus reconstructed via the $\tau$-function is learnable, even under noisy scenarios where Socrates is not perfectly accurate. 
To do this, we will use our main technical lemma, \lemref{tauinvariant}, stated informally below:

\begin{inflemma}[\textbf{\ref{lem:tauinvariant}, Informal}]

Suppose $E$ is sampled i.i.d. from $\P$, and $(\epsilon, \delta, W)$-representative of $\P$ for some hypothesis class $W$ with VC dimension $d$, and some $\epsilon$ and $\delta$. 
Also suppose $D$ is not $(\epsilon, \delta, W)$-representative of $\P$ and $D \cap E = \emptyset$.

Then, if some technical conditions hold
then the dataset 
\begin{equation}
\tilde{D} = D \cup \{\langle \tilde{x}, \mathcal{O}(\tilde{x}; w^o) \rangle \colon \tilde{x} = \tau_\rho(x) \land \tilde{x} \not\in E \}
\end{equation}

\noindent is $(\epsilon, \delta, W)$-learnable, for any $x \in E$, some oracle $\mathcal{O} \colon X \times \{w^o\} \rightarrow Y$, and  $\tau$-function $\tau_\rho(\cdot)$. 
\end{inflemma}

\lemref{tauinvariant} shows that for suitable values of $\epsilon, \delta$, and $W$, a (training) dataset $D$ can be made $(\epsilon, \delta, W)$-learnable by generating points with a $\tau$-function, and sampling the evaluation set $E$. 
The technical conditions relate to the lower bounds of $\size{D}, \epsilon$ and $\delta$, as well as sufficient sampling iterations for $\tau$ to guarantee that reconstruction up to homology of the underlying (compact) Riemannian submanifold ensures learnability of the dataset. 

The core contribution of \lemref{tauinvariant} is then to bridge the gap between minimum sizes for learnability and minimum sizes for reconstructability of a manifold. 

With \lemref{tauinvariant}, and replacing the oracle with a noisy model like Socrates, we then prove the bounds from \eqref{ineqbounds1} in \lemref{periterationperf}. \eqref{hpobounds1} follows from \lemref{hyperparamremoval}. 

Hence this section is structured as follows: we begin by proving our main technical lemma (\secref{tauproofs}), and then use it to prove \lemstworef{periterationperf}{hyperparamremoval} (\secstworef{periterationperfproof}{hporemovalproof}). 
We then state and prove \thmref{maintheorem} in \secref{mainthmproof}.

\subsection{Proof of the Main Technical Lemma}\label{sec:tauproofs}

In this section we prove our main technical lemma, \lemref{tauinvariant}. 
Informally, it states that, for suitable values of $\epsilon, \delta$, and $W$, a dataset $D$ can be made $(\epsilon, \delta, W)$-learnable by generating points with a $\tau$-function, and sampling the evaluation set $E$. %

Prior to that, we provide some intuition behind \lemref{tauinvariant} and its supporting lemmas. 
First we show in \lemref{learnablesamplesize} that an open cover for $D$ is equivalent up to homology to the manifold $\X$ where $\P$ is supported on, iff $D$ is $(\epsilon, \delta, W)$-representative of $\P$. 
We extend this result in \lemref{homologouslemma} to show that if the covers of two distinct datasets are equivalent up to homology, then they are both $(\epsilon, \delta, W)$-representative of $\P$ if one of them also fulfills \lemref{learnablesamplesize}. 
Then we show in \lemref{tauworks} that reconstructing $\X$ up to homology is possible by perturbing the data from a known sample of said manifold---the validation set $E$. 
The proof for \lemref{tauinvariant} then follows. %

\begin{lemma}\label{lem:learnablesamplesize}
Let $D = \{\langle x_i, y_i \rangle \colon x_i \sim \P, y_i = c(x_i) \}_{i \in [1, n]}$ 
be a dataset sampled from a probability distribution $\P$ supported on a compact Riemannian submanifold $\X \subset \mathbb{R}^m$ with reach $\mu > 0$. 
Let $W$ be a hypothesis class with VC dimension $d$. 

Let $\rho$ correspond to a choice of $\rho/4$-covering number for $\X$ such that $0 < \rho < \mu/2$, and let
\begin{equation}
\Lambda_\rho = \left(\frac{\rho \sqrt{\pi}}{4}\right)\left(1 - \frac{\rho^2}{64\mu^2} \right)^{1/2}.
\end{equation}

Fix a $0 < \delta \leq 1/2$. 

Then, if 

\begin{align}
\epsilon &\in \BigOmega {\left(\frac{\Lambda_\rho^m}{(\frac{m}{2})!\vol{(\X)}}\right)^{1/2}}, \label{eq:lemmahomologyepsilon}\\
d &\in  \BigOmega{\log{\left(2^m\left(\frac{64 - \rho^2/\mu^2}{256 - \rho^2/\mu^2}\right)^{m/2}\frac{(\frac{m}{2})!\vol{(\X)}}{\Lambda_\rho^m}\right)}}, \label{eq:lemmahomologyvc}\\
\size{D} &\in \BigOmega{\frac{d + \log{(1/\delta)}}{\epsilon^2}} \label{eq:lemmahomologyd},
\end{align}

\noindent the homology of an open subset $\cup_{\langle x_i, y_i \rangle \in D} \B_\rho(x_i)$ equals the homology of $\X$, if and only if $D$ is $(\epsilon, \delta, W)$-representative of $\P$, with probability at least $1 - \delta$.

\end{lemma}
\begin{proof}
In \appref{homologyproof}. 
\end{proof}

\begin{lemma}\label{lem:homologouslemma}
Let $D = \{\langle x_i, y_i \rangle \colon x_i \sim \P, y_i = c(x_i) \}_{i \in [1, n]}$, $E = \{\langle x_i, y_i \rangle \colon x_i \sim \P, y_i = c(x_i) \}_{i \in [1, p]}$ be two datasets sampled from a probability distribution $\P$ supported on a compact Riemannian submanifold $\X \subset \mathbb{R}^m$ with reach $\mu > 0$. 

Assume $D$ is $(\epsilon, \delta, W)$-representative of $\P$ with a set of classifiers $W$ with VC dimension $d$ for some $\rho$, $d$, $\epsilon$ and $\delta$. %

Let $U_D = \cup_{\langle x_i, y_i \rangle \in D} \B_\rho(x_i)$ and $U_E = \cup_{\langle x_i, y_i \rangle \in E} \B_\rho(x_i)$ be two open covers of $\X$ with elements of $D$ (resp. $E$). 

Then, if the homology of $U_D$ equals the homology of $U_E$, then $E$ is also $(\epsilon, \delta, W)$-representative of $\P$ for the same set of classifiers, and values of $\rho$, $d$, $\epsilon$, and $\delta$.
\end{lemma}
\begin{proof}
It follows immediately from \lemref{learnablesamplesize}.
\end{proof}

\begin{lemma}\label{lem:tauworks}
Let $E = \{\langle x_i, y_i \rangle \colon x_i \sim \P, y_i = c(x_i) \}_{i \in [1, p]}$ be a dataset sampled i.i.d. from a probability distribution $\P$, whose support $X$ lies on a compact Riemannian submanifold $\X \subset \mathbb{R}^m$ with reach $\mu > 0$. 

Suppose $E$ is $(\epsilon, \delta, W)$-representative of $\P$ with a set of classifiers $W$ of VC dimension $d$, where $d$, $\size{E}$, $\delta$ and $\epsilon$ are as stated in \lemref{learnablesamplesize}. 

Let $\tau^{Z}_\rho(\cdot)$ be a $\tau$-function parametrized by $0 <\rho < \mu/2$, and assume that the domain $Z$ of $\tau^{Z}_\rho(\cdot)$ lies on a compact Riemannian submanifold $\Z \subset \mathbb{R}^m$ with the same reach $\mu > 0$, and such that $\X \subset \Z$. 

Let $\Delta = Z \setminus X$, and let $\tilde{E}^{(\kappa)}$ be a set constructed by sampling every element of $E$ with the $\tau$-function $\kappa$ times, 

\begin{equation}\label{eq:tildeedef}
\tilde{E}^{(\kappa)} = \bigcup_{j=1}^\kappa \{\tau^{Z}_\rho(x) \colon \langle x, y \rangle \in E \land \tau^{Z}_\rho(x) \not\in E\}.
\end{equation}

Then if 

\begin{align}
2\size{\Delta} &< \size{X} - 3 \size{E}\text{ and}\\
\kappa &\geq \left(\frac{\size{X} + \size{\Delta}}{\size{X} - \size{E}}\right) \ln{\left(\frac{\size{E}}{\delta}\right)},
\end{align}

\noindent $\tilde{E}^{(\kappa)}$ is $(\epsilon, \delta, W)$-representative of $\P$ with probability at least $1 - \delta$.
\end{lemma}
\begin{proof}
If an open cover of $\tilde{E}^{(\kappa)}$ has the same homology as an open cover of $E$, by \lemref{homologouslemma} this set is $(\epsilon, \delta, W)$-representative of $\P$. 
A full proof is in \appref{tauworksproof}. 
\end{proof}

We may now state and prove our main technical lemma:

\begin{lemma}\label{lem:tauinvariant}
Let $D = \{\langle x_i, y_i \rangle \colon x_i \sim \P, y_i = c(x_i) \}_{i \in [1, n]}$, $E = \{\langle x_i, y_i \rangle \colon x_i \sim \P, y_i = c(x_i) \}_{i \in [1, p]}$ be two datasets sampled from a probability distribution $\P$ supported on a compact Riemannian submanifold $\X \subset \mathbb{R}^m$ with reach $\mu > 0$. 

Suppose $E$ is sampled i.i.d. from $\P$, and $(\epsilon, \delta, W)$-representative of $\P$ for some hypothesis class $W$ with VC dimension $d$, and some $\epsilon$ and $\delta$. 
Also suppose $D$ is not $(\epsilon, \delta, W)$-representative of $\P$ and $D \cap E = \emptyset$. 

Let $\mathcal{O} \colon X \times \{w^o\} \rightarrow Y$ be an oracle with VC dimension $\size{X}$, and  
let $\tau^{Z}_\rho(\cdot)$ be a $\tau$-function parametrized by $0 <\rho < \mu/2$; where $\rho$ corresponds to the minimum $\rho/4$-covering number of $\X$, and such that $Z$ lies on $\X$. 

Then, if the conditions from \lemstworef{homologouslemma}{tauworks} hold; and $d$, $\size{E}$, $\delta$ and $\epsilon$ are as stated in \lemref{learnablesamplesize}, then the dataset 
\begin{equation}
\tilde{D} = D \cup \{\langle \tilde{x}, \mathcal{O}(\tilde{x}; w^o) \rangle \colon \tilde{x} = \tau^{Z}_\rho(x) \land \tilde{x} \not\in E \}
\end{equation}

\noindent is $(\epsilon, \delta, W)$-learnable, for any $x \in E$. 
\end{lemma}

We are now ready to complete the proof of \lemref{tauinvariant}: 
\begin{proof}

It follows from \lemref{tauworks}. 
If $\epsilon, \delta$, and $d$ remain fixed, then the only variable required to obtain $(\epsilon, \delta, W)$-representativeness of $\P$ is the size of $\tilde{D}$. 
Suppose, without loss of generality, $\size{D} = \size{E} - 1$. 

Given that the $\tau$-function has its domain lying on the same manifold as $X$, the first sample of $\tau^{Z}_\rho(x)$ belongs to $\P$, and is correctly labeled by $\mathcal{O}(\tau^{Z}_\rho(x);w^o))$. 

The representativeness of this dataset follows from \lemstworef{learnablesamplesize}{homologouslemma}. Let $F = \{\langle \tilde{x}, \mathcal{O}(\tilde{x}; w^o) \rangle \colon \tilde{x} = \tau^{Z}_\rho(x) \land \tilde{x} \not\in E \}$. 
By \lemref{homologouslemma}, a cover of $\tilde{D} = D \cup F$ has the same homology as a cover of $E$, and by \lemref{learnablesamplesize} $\tilde{D}$ is $(\epsilon, \delta, W)$-representative of $\P$. %

Moreover, note that the VC dimension of $W$ remains unchanged, but adding a new point to $D$ lowers the Rademacher complexity of this problem,

\begin{equation}
\sqrt{\frac{2d\ln{(\size{D} + \size{F})}}{\size{D} + \size{F}}} < \sqrt{\frac{2d\ln{(\size{D})}}{\size{D}}}.
\end{equation}

\noindent This concludes the proof.

\end{proof}

It is important to highlight that equivalence up to homology is a weak relation, in comparison to, say, equivalence up to homeomorphism. Unfortunately, this is an undecidable problem for dimensions larger than $4$ \cite{Markov}.

\subsection{Proof of \lemref{periterationperf}}\label{sec:periterationperfproof}
Prior to proving our lemma, we introduce one additional definition:

\begin{definition}[Perfect Memory of a Classifier]\label{def:perfectmem}
Let $D \subset X \times Y$ be a dataset, $\theta \subset \mathbb{R}$ a hyperparameter set, and $f \colon X \times W \rightarrow Y$ a classifier, such that they are inputs to a training algorithm \train$ \colon W \times X \times Y \times \theta \rightarrow W $. 
Suppose the algorithm returns a weight set $w^D$, and acc$(f(\cdot; w^D), D) = a$. 

Let $D' = D \cup \{\langle x, y \rangle\}$ for any $\langle x, y \rangle$ from $X\times Y$, such that $\langle x, y \rangle\not\in D$. 
Let $w^{D'} = \train(f, D', \theta)$, and $\text{acc}(f(\cdot; w^{D'}), D') = a'$. 

We say that $f(\cdot; \cdot)$ has \emph{perfect memory with} \train$(\cdot, \cdot, \cdot)$ \emph{and} $\theta$ if the following relation holds:

\begin{equation}
a - \frac{1}{\size{D} + 1} \leq a' \leq a + \frac{1}{\size{D} + 1}.
\end{equation}

\end{definition}

\defref{perfectmem} is needed to simplify the proofs, and, informally, it means that the accuracy of a given model on the previously seen points will not change when augmenting the dataset and retraining it. 
Although it is mostly a theoretical object, applied analogies can be found in models designed to overcome catastrophic forgetting (for example, with elastic weight consolidation \cite{Catastrophic}); or, more practically, these with sufficient capacity to learn the entire distribution. 
We discuss in detail in \secref{conclusion} to which extent this assumption is reasonable.

\begin{lemma}\label{lem:periterationperf}
Let $\langle f, S, D, E, \Theta, \tau^{Z}_\rho \rangle$ be an input to Agora, where $D$ and $E$ are two datasets  $D = \{\langle x_i, y_i \rangle \colon x_i \sim \P, y_i = c(x_i) \}_{i \in [1, n]}$, $E = \{\langle x_i, y_i \rangle \colon x_i \sim \P, y_i = c(x_i) \}_{i \in [1, p]}$ sampled from a probability distribution $\P$ whose support $X$ lies on a compact Riemannian submanifold $\X \subset \mathbb{R}^m$ with reach $\mu > 0$. 

Suppose $f(\cdot; \cdot)$ has perfect memory with \train$(\cdot, \cdot, \cdot)$ and all $\theta \in \Theta$. 
Assume that $Z$ lies on $\X$, and that $\rho$ corresponds to the minimum $\rho/4$-covering number of $\X$. 

Assume $D$ and $E$ have equal proportion of positive and negative labels. 
If Socrates has accuracy acc$(S(\cdot; w^S),F) \geq 2/3$, $\forall F \subset X \times Y$, 
then the accuracy of $f(\cdot; \cdot)$ on $E$ at the end of the $k^{\text{th}}$ iteration of Agora is given by:

\begin{equation}
\text{acc}^{(k)}(f(\cdot; w^{*, (k)}), E) \geq 1 - \frac{1}{2^k} .
\end{equation}

\end{lemma}
\begin{proof}
By \lemref{tauinvariant}, under the lemma's conditions augmenting the dataset with points similar to the ones in $E$ improves the performance of $f(\cdot; \cdot)$. The per-iteration accuracy and bounds for Socrates follow by induction on the number of mislabeled points by a randomized version of $f(\cdot; \cdot)$ on $E$. 
A full proof is in in \appref{periterationperfproof}.
\end{proof}

\subsection{Proof of \lemref{hyperparamremoval}}\label{sec:hporemovalproof}
\begin{lemma}\label{lem:hyperparamremoval}
Let $\langle f, S, D, E, \Theta, \tau^{Z}_\rho \rangle$ be an input to Agora, where $D$ is a dataset  $D = \{\langle x_i, y_i \rangle \colon x_i \sim \P, y_i = c(x_i) \}_{i \in [1, n]}$ sampled from a probability distribution $\P$ supported on a compact Riemannian submanifold $\X \subset \mathbb{R}^m$ with reach $\mu > 0$. 

Suppose that Timaeus has perfect memory with \train$(\cdot, \cdot, \cdot)$ and all $\theta \in \Theta$. 
Assume that $Z$ lies on $\X$, and that $\rho$ corresponds to the minimum $\rho/4$-covering number of $\X$. 

Assume $D$ and $E$ have equal proportion of positive and negative labels. 

Let $\theta^*$ be the final hyperparameter set for the entire run of Agora, with a corresponding trained model $f(\cdot; w^* \in W)$. Let $D^*$ be the final training set. 
Then, for any $\theta_{i} \in \Theta$, if Socrates has accuracy acc$(S(\cdot, w^S), F) \geq 2/3$, $\forall F \subset X \times Y$, 
the following holds:

\begin{equation}
\text{acc}(f(\cdot; w^*), E) \geq \text{acc}(f(\cdot; w_{i}), E),
\end{equation}

\noindent where $w_{i} \in W$ (resp. $w^*$) is the output of \train$(f, D^*, \theta_{i})$ (resp. $\theta^*$). 
\end{lemma}
\begin{proof}
In \appref{hyperparamremovalproof}, by an application of \lemref{periterationperf}. 
\end{proof}

\subsection{Proof of \thmref{maintheorem}}\label{sec:mainthmproof}

We are now ready to state and prove our main result:

\begin{theorem}\label{thm:maintheorem}
Let $\langle f, S, D, E, \Theta, \tau^{Z}_\rho \rangle$ be an input to Agora, where $D$ and $E$ are two datasets $D = \{\langle x_i, y_i \rangle \colon x_i \sim\P, y_i = c(x_i) \}_{i \in [1, n]}$, $E = \{\langle x_i, y_i \rangle \colon x_i \sim\P, y_i = c(x_i) \}_{i \in [1, p]}$  sampled from a probability distribution $\P$ whose support $X$ lies on a compact Riemannian submanifold $\X \subset \mathbb{R}^m$ with reach $\mu > 0$.  

Suppose $f \colon X \times W \rightarrow Y$ has perfect memory with \train$(\cdot, \cdot, \cdot)$ and all $\theta \in \Theta$. 
Assume the domain of the $\tau$-function lies on $\X$, that $\rho$ corresponds to the minimum $\rho/4$-covering number of $\X$, and that Socrates has a lower-bound acc$(S(\cdot; w^S), F) \geq 2/3$ for all $F \subset X \times Y$.

Let $\text{acc}(f(\cdot; w^*), E)$ be the accuracy of a model on $E$, where $w^* \in W$ is the result of running Agora with this input.

Let $\text{acc}(f(\cdot; w^e), E) = \mmax{\theta \in \Theta}\{acc(f(\cdot; w), E) \colon w = \train(f, D, \theta)\}$ be the accuracy of a model obtained from an enumeration over all $\Theta$. 

Then, if $D$ and $E$ have equal proportion of positive and negative labels, the following holds:

\begin{equation}\label{eq:ineqbounds}
\text{acc}(f(\cdot; w^{*}), E) = r \cdot \text{acc}(f(\cdot; w^e), E),
\end{equation}

\noindent for $1 \leq r \leq 2\left(1 - \frac{1}{2^{\size{\Theta}}}\right)$. 

Moreover, let $D^*$ and $\theta^*$ be the final dataset and hyperparameter set from running Agora with this instance. Then the accuracy of the returned Timaeus model $f(\cdot; w^*)$ on $E$ upper bounds all the other hyperparameter sets $\theta_i \in \Theta$, when trained with the final dataset $D^*$, 

\begin{equation}\label{eq:hpobounds}
\text{acc}(f(\cdot; w^*), E) \geq \mmax{\theta \in \Theta \setminus \{\theta^*\}}\{\text{acc}(f(\cdot; w), E) \colon w = \train(f, D^*, \theta)\}. 
\end{equation}
\end{theorem}
\begin{proof}
\eqref{ineqbounds} follows from \lemref{periterationperf}. 
It is equivalent to noting that an enumeration is equivalent to the first iteration of Agora, and using \lemref{periterationperf} to assert that further iterations improve the accuracy of the Timaeus model. 
\eqref{hpobounds} follows immediately from \lemref{hyperparamremoval}. %
\end{proof}

Note that the bounds obtained in \thmref{maintheorem} are fairly loose, and can be tightened on a model-by-model basis. 

\begin{remark}
In practice we might not have access to the conditions from \lemref{tauinvariant}---namely, a $\tau$-function with $\Z = \X$ and $\rho$ being the minimum $\rho/4$-covering number of $\X$, as well as an oracle. 
In the case where the $\tau$-function does not fulfill these characteristics, the bounds still hold as long as the conditions of \lemref{tauworks} are met, as it is a probabilistic argument bounded by $\kappa$. 
The lemmas used to prove the main theorem do not assume Socrates to be an oracle. 
\end{remark}
\begin{remark}
One implication of \lemref{tauworks} is that, if $\Z = \X$, the number of calls to $\tau^{Z}_\rho(\cdot)$ is minimal for any input problem, since, by definition, it will construct the minimum-sample size dataset required for $(\epsilon, \delta, W)$-learnability of $\P$ with the chosen classifier $f(\cdot; \cdot)$. 
\end{remark}

\section{Time Bounds}\label{sec:timebounds}

In this section we provide runtime bounds for Agora. 
We show in \thmref{timecomplexity} that the runtime for Agora under the conditions from \thmref{maintheorem} is polynomial on \train$(\cdot, \cdot, \cdot)$, the cardinalities of $\size{\Theta}$ and $\size{E}$, and the inference times for Timaeus and Socrates. We conclude this section by narrowing down these results in \corref{sgdcor} for a class of inputs that rely on stochastic gradient descent (SGD).

For notational simplicity, we use $\bar{f}$ to denote the number of steps required to obtain an output from a function $f(\cdot; \cdot)$ given a fixed-size input $x$ regardless of other parameters, $\bar{f} = \BigTheta{f(x; w)}$, $\forall w \in W$. 

\begin{theorem}\label{thm:timecomplexity}
Let $\langle f, S, D, E, \Theta, \tau^{Z}_\rho\rangle$ be an input to Agora, where $D$ and $E$ are two datasets $D = \{\langle x_i, y_i \rangle \colon x_i \sim \P, y_i = c(x_i) \}_{i \in [1, n]}$, $E = \{\langle x_i, y_i \rangle \colon x_i \sim \P, y_i = c(x_i) \}_{i \in [1, p]}$ sampled from a probability distribution $\P$ whose support $X$ lies on a compact Riemannian submanifold $\X \subset \mathbb{R}^m$ with reach $\mu > 0$. 

Suppose $f \colon X \times W \rightarrow Y$ has perfect memory with \train$(\cdot, \cdot, \cdot)$ and all $\theta \in \Theta$, and that acc$(S(\cdot; w^S), F) \geq 2/3$, $\forall F \subset X \times Y$. 
Also assume $Z$ lies on $\X$, $\rho$ corresponds to the minimum $\rho/4$-covering number of $\X$, and that the conditions from \lemref{tauworks} hold. 

Then if the runtime of \train$(\cdot, \cdot, \cdot)$ is bounded by $\operatorname{O}(\text{poly}(\bar{f}, \size{D}, \theta))$ and $D$ has an equal ratio of positive and negative labels, Agora terminates in 

\begin{equation}\label{eq:runtimecomplexity}
\BigO{\size{\Theta}^2\left(\log{\size{\Theta}} + 
\size{\theta}^2 + T_f\left(\size{D} + \frac{\size{E}}{2^{\size{\Theta}-1}}\right)\right) + 
\bar{S}\size{\Theta}\size{E}\left(1 - \frac{1}{2^{\size{\Theta}}}\right)}
\end{equation}

\noindent steps, where $T_f(\size{D}) \in \operatorname{\Omega}(\text{poly}(\bar{f}, \size{D}, \theta))$, $\forall \theta \in \Theta$ and $f(\cdot; \cdot)$. 

\end{theorem}
\begin{proof}
In \appref{timecomplexityproof}.
\end{proof}

So far we have assumed that the hypothesis class is able to learn, with probability $1 - \delta$, $\P$ in a polynomial number of steps. This may be an unrealistic assumption given that there is no guarantee that exists an algorithm fulfilling the conditions for $\train(\cdot, \cdot, \cdot)$.

We now focus on the case where Timaeus is a piecewise-continuous, real-valued function with its output mapped consistently to $\{0,1\}$. 
Due to the nonconvexity and discontinuities present in \eqsref{errorfunction}{accuracyfunction}, a carefully-chosen surrogate loss\footnote{See, for example Bartlett et al. \newcite{BartlettAndJordan}, and Nguyen et al. \newcite{NguyenWainwrightJordan}.} is often used along with a polynomial-time (or better) optimization algorithm, such as SGD.

If we choose the optimizer for \train$(\cdot, \cdot, \cdot)$ to be SGD, and the surrogate loss is an $L$-Lipschitz smooth function, bounded from below and with bounded stochastic gradients for some $G$, then \train$(\cdot, \cdot, \cdot)$ returns an $\epsilon$-accurate solution in $\operatorname{O}(\size{B}\bar{f}/\epsilon^c)$ steps; where $c \leq 2$ is a constant, $\size{B}$ is the batch size for some $B \subset D$, and $\eta$ is the learning rate. 
Most SGD-based training procedures present such a runtime, even in the non-convex setting \cite{UnifiedNguyen,Reddy}. Nguyen et al. \newcite{UnifiedNguyen} showed that, for a specific variant of SGD referred to as \emph{shuffling-type SGD}, this algorithm converges to some $\mathbb{E}[\size{\size{\nabla F(\cdot; w)}}^2] \leq \gamma$ with $c = 2/3$ whenever $\eta = \sqrt{\gamma}/(LG)$, for some (not necessarily convex) function $F$. Note that $\gamma$ may be a stationary point, and that, for an appropriate choice of loss function, $\gamma \propto \epsilon$. 

\begin{corollary}\label{cor:sgdcor}
Assume Timaeus is a piecewise-continuous, real-valued function, and that the rest of the conditions for \thmref{timecomplexity} hold. Also assume that \train$(\cdot, \cdot, \cdot)$ uses shuffling-type SGD as its optimizer.

Suppose that for every $\theta \in \Theta$ there exists a subset that encodes the learning rate $\eta$ and batch size $\size{B}$ for \train$(\cdot, \cdot, \cdot)$, $\{\size{B}_i, \eta_i\} \subset \theta_i;\; \forall \theta_i \in \Theta$. 
If the surrogate loss is is $L$-Lipschitz smooth, bounded from below and where all gradients are stochastic and bounded by some $G$, 
then Agora converges in 

\begin{equation}\label{eq:nnruntimecomplexity}
\BigO{\size{\Theta}^2\left(\log{\size{\Theta}} + 
\size{\theta}^2 + 
\bar{f}\left(\size{E}+ \frac{\size{B}}{(LG\zeta)^{2/3}}\right)\right) + 
\bar{S}\size{\Theta}\size{E}\left(1 - \frac{1}{2^{\size{\Theta}}}\right)}
\end{equation}
\noindent steps, where $B \subset D$, with $\size{B} = \mmax{\theta \in \Theta}\{\size{B} \in \theta\}$ and $\zeta = \mmin{\theta \in \Theta}\{\eta^2 \in \theta\}$.
\end{corollary}
\begin{proof}
In \appref{nnruntimecomplexityproof}.
\end{proof}

\corref{sgdcor} is not necessarily applicable to all neural networks, since their convergence conditions strongly depend on the choices of optimizers, architecture, and hyperparameters. The ratio from \thmref{maintheorem} is hence conditional on $\Theta$, as neural networks, without accounting for finite-precision computing, belong to hypothesis classes of infinite size.

\subsection{Discussion}\label{sec:discussioncomp}

The main advantage of Agora comes from its data expansion step; from \thmref{maintheorem} we can see that it enables the Timaeus model to learn a training set that otherwise wouldn't be learnable with an enumeration of the search space. It is also efficient in terms of the needed sample size, as implied by the proof of \lemref{tauinvariant}. 

However, it can be seen from \thmref{timecomplexity} Agora is a costly algorithm to run. It is considerably slower than an enumeration of the search space---that is, picking the best $w$ from $\{ \text{acc}(f(\cdot; w), E) \colon w = \train(f, D, \theta)\; \forall \theta \in \Theta\}$. 

This can be seen from \linestworef{startfor}{endfor}. The first iteration of Agora is precisely this enumeration, and it runs in $\BigO{\size{\Theta}T_f(\size{D}) + \bar{f}\size{E}}$ steps. 
Agora executes these lines at most $\size{\Theta}$ times, in addition to having two terms dependent on the call to Socrates and the data expansion. The hyperparameter pruning step also contributes another $\log{\size{\Theta}}$ term. 

That being said, by our problem setting, an enumeration would be unable to learn the data satisfactorily. It follows that Agora is a good choice of algorithm when the data is not representative of the problem at hand, and unwise otherwise. 
In \appref{experiments} we perform a brief experimental comparison of Agora and an enumeration, and its outcome aligns with our theoretical predictions around both runtime and correctness.

\section{Concluding Remarks}\label{sec:conclusion}

Agora is an algorithm to obtain models under low-resource and ill-posed settings. It works by iteratively reconstructing the manifold where the underlying distribution is supported on, based on the mistakes done by the learner on the evaluation set. 
We show that this approach is provably able to generate a dataset that is representative of the distribution, which makes the problem easier to learn. 
More generally speaking, we showed under which conditions data expansion techniques, abstracted out here in the $\tau$-function, are effective.

A central part of our correctness results is the assumption that the support is sufficiently dense, and lies on or around a manifold with positive reach. These are fairly strong constraints---perhaps even unrealistic. The work by Genovese et al. \newcite{GenoveseEtAl} shows that a well-defined ridge can be topologically similar, in a certain sense, to said manifold. Similar to Niyogi et al. \newcite{NiyogiSmale}, this result is also resilient to sampling noise, and could be a stronger basis from which to characterize the $\tau$-function. While the constraint around the reach can be also weakened, it is known that a smooth compact submanifold of $\mathbb{R}^m$ always has positive reach \cite{Chazal,ChazalCurvature}. 
The density of the support is crucial for the results from this paper, as it is for all manifold learning methods. Density can be achieved by a dimensionality reduction step, which has the benefit of re-establishing the performance bounds at the cost of some preprocessing overhead.

Further work could deal with expansion of this work to other problems, as we primarily deal with binary classification; 
or applying our theoretical results to other areas of machine learning. 
For example, the correspondence between learnability and reconstructability of a manifold may be extended to neural network compression by setting the problem over a neuromanifold (see, e.g., \cite{amari}), and rephrasing network pruning as a manifold reconstructability problem. 

\begin{acknowledgements}
The author would like to thank the anonymous referees for their helpful comments, which improved the quality of this paper, and Q. Wang for their feedback on a draft version of this work. 
This paper is dedicated to the memory of the author's grandfather.
\end{acknowledgements}

\section*{Data Availability Statement}
This paper's experimental results use the Breast Cancer Wisconsin (Diagnostic) Dataset \cite{UWDataset}, freely available in UCI's Machine Learning Repository \cite{Dua}. 

\section*{Conflict of interest}
The author is affiliated with Microsoft Corporation and the University of York. 
This work was done while at Amazon Alexa and the University of Texas at Austin.

\bibliography{biblio}

\newpage
\appendix
\section*{Appendices}

\section{Proof of \thmref{maintheorem}}\label{app:maintheoremproof}

The proof for \eqref{ineqbounds} follows from \lemref{periterationperf}. 
\lemref{periterationperf} gives a closed-form solution to a lower-bound on the per-iteration accuracy of the model. We solve this equation for the worst-case scenario in Agora, when it runs for $\size{\Theta}$ iterations, and an enumeration. This can be done since \linestwosref{startfor}{endfor} in Agora are equivalent to said enumeration. %

Concretely, the right-hand side bound in \eqref{ineqbounds} is the worst-case scenario for Agora, that is, when $k = \size{\Theta}$. 
The left-hand side bound follows the case where $k=1$, since it is an enumeration. 
Hence:
\begin{align}
\text{acc}(f(\cdot; w^e), E) &\geq 1 - \frac{1}{2}, \\
\text{acc}(f(\cdot; w^{*}), E) &\geq 1 - \frac{1}{2^{\size{\Theta}}},
\end{align}
and therefore 
\begin{equation}
r = \frac{\text{acc}(f(\cdot; w^{*}), E)}{\text{acc}(f(\cdot; w^e), E)} \leq 2\left(1 - \frac{1}{2^{\size{\Theta}}}\right),
\end{equation}
as desired. 

\eqref{hpobounds} follows immediately from \lemref{hyperparamremoval}.

\section{Proof of \lemref{periterationperf}}\label{app:periterationperfproof}

We begin by lower-bounding the performance of a classifier with the notion of a \emph{random classifier}: 

\begin{definition}[The Random Classifier]\label{def:randomclassifier}
Let $D = \{\langle x_i, y_i \rangle \colon x_i \sim \P, y_i = c(x_i) \}_{i \in [1, n]}$ be a dataset sampled from $\P$ with $p$ positive labels and $1-p/\size{D}$ negative labels. A \emph{random classifier} $f$ outputs, when called, $1$ with probability $p/\size{D}$, and $0$ otherwise, regardless of its input.
\end{definition}

\defref{randomclassifier} is a lower-bound on the performance of any other model, as given by the training algorithm. This holds as long as the data remains balanced.\footnote{This statement does not hold in all cases. If the data were imbalanced across all $\P$, "dumber" models could actually perform better, e.g., guessing all zeros in a problem with $\size{X} - 1$ negative labels.}

Assume, for simplicity, $\Theta = \{\theta_1, \dots, \theta_k \colon \theta_i \cap \theta_j = \emptyset\; \forall \theta_i, \theta_j \}$, $f(\cdot; \cdot)$ is a random classifier, and $W = \{w\}$. 
The cases where $W$ is not a singleton or $f(\cdot; w)$ is not a random classifier follow immediately from \linestworef{maxstatement}{removestatement} and the proof for this base case.

Let $D^{(1)}$ be the dataset at the beginning of the first iteration of Agora. 
Since $f(\cdot; w)$ is a random classifier, its accuracy on $E$ is at least $a^{(1)}_f \geq 1/2$, and hence its error rate is $e^{(1)}_f \leq 1/2$.  

Let the dataset constructed by \linestworef{sortstatement}{checkstatement} be $M^{(1)}$. 
Then $\size{M^{(1)}} \leq \size{E}/2$, and hence 

\begin{equation}\label{eq:firststep}
\size{D^{(2)}} \leq \size{D} + \frac{\size{E}}{2}. 
\end{equation}

At the next iteration, the labels of every element in $\size{M^{(2)}}$ have a chance acc$(S(\cdot; w^S))$, $M^{(2)}) \geq 2/3$ of being correct, and hence the number of points that are misclassified is at most $\size{E \cap M^{(1)}}/2$, or

\begin{equation}\label{eq:secondstep}
\size{M^{(2)}} \leq \frac{\size{E}}{4},
\end{equation}

\noindent since the $\tau$-function has $\Z = \X$ and $\rho$ corresponds to the minimum $\rho/4$-covering number of $\X$. Hence each $\langle x_i, y_i \rangle \in E$ is classified correctly with probability $\geq 1/3$ in the case where Socrates correctly predicted the label for its corresponding point $\langle \tilde{x}_i, S(\tilde{x}_i; w^S)\rangle \in D^{(2)}$.
If, for some $\langle x_i, y_i \rangle \in E$, $f(x_i; w) = y_i$, but its corresponding point was incorrectly labeled by Socrates, it will not appear again in any $M^{(k)}$ due to the perfect memory of $f(\cdot; w)$. 

Generalizing the above, it can be seen that at any iteration $k$ the size of $M^{(k)}$ is given by 

\begin{equation}\label{eq:generalstep}
\size{M^{(k)}} \leq \frac{\size{E}}{2^k}.
\end{equation}

Induction on \eqref{generalstep} concludes the proof.

\section{Proof of \lemref{hyperparamremoval}}\label{app:hyperparamremovalproof}
By an application of Agora and \lemref{periterationperf}. Assume, for simplicity, that Timaeus is a random classifier as defined in \appref{periterationperfproof}.  

To begin, suppose that we run the algorithm on a separate instance $\langle f, S, D, E, \{\theta_{1}, \theta_{2}\}, \tau^{Z}_\rho \rangle$, such that $\size{\theta_{1} \cap \theta_{2}} \leq \min{(\size{\theta_1}, \size{\theta_2})} - 1$. 
Agora only runs for $k=2$ iterations, evaluating three models. At the first iteration, it obtains two models corresponding to $\theta_1$ and $\theta_2$, $f(\cdot; w^{(k=1)}_{1})$ and $f(\cdot; w^{(k=1)}_{2})$, and prunes out the lowest performing hyperparameter set---say, $\theta_{1}$. Let $f(\cdot; w^{*, (1)}_{2})$ be the trained model at $k=1$.

At the second iteration, it evaluates only one model, $f(\cdot; w^{(2)}_{2})$. 
Call the trained version $f(\cdot; w^{*, (2)}_{2})$. 
Since $f(\cdot; \cdot)$ has perfect memory, by \lemref{periterationperf} and \linestworef{maxstatement}{removestatement} we know that

\begin{equation}
\text{acc}(f(\cdot; w^{*, (1)}_{1}), E) \leq \text{acc}(f(\cdot; w^{*, (1)}_{2}), E) \leq \text{acc}(f(\cdot; w^{*, (2)}_{2}), E),
\end{equation}

\noindent and hence $\theta^{(2)} = \theta_2$ is the convergent (final) hyperparameter set for Agora for this input instance. 

Since the accuracy of $f(\cdot; \cdot)$ does not decrease when calling Socrates, we see that, across the entire run of the algorithm, the convergent hyperparameter set is not be removed. 
In other words, if $\theta_j$ is the aforementioned set, and $k$ the total number of iterations, for any $i < k$, 
\begin{equation}
\text{acc}(f(\cdot; w^{*, (k-i)}_{l}), E) \leq \text{acc}(f(\cdot; w^{*, (k)}_{j}), E),
\end{equation}

\noindent for any $\theta_{l} \in \Theta$, $\theta_{l} \neq \theta_{j}$. 

It follows that the accuracy of $f(\cdot; \cdot)$ on the convergent hyperparameter set, for any instance of Agora, upper bounds the other members of $\Theta$.

\section{Proof of \lemref{learnablesamplesize}}\label{app:homologyproof}

\paragraph{"If" direction:}
We show that an open set $U = \cup_{\langle x_i, y_i \rangle \in D} \B_\rho(x_i)$ that has the same homology as $\X$ is $(\epsilon, \delta, W)$-representative of $\P$, and that \eqref{errorprobfunction} holds when the conditions for the lemma are met. 

Fix an arbitrary confidence $0 < \delta \leq 1/2$. For notational convenience, let $D_X = \{x_i \colon \langle x_i, y_i \rangle \in D\}$. 

It was shown by Niyogi et al. \newcite{NiyogiSmale} that the minimum sample size $\size{D} = \size{D_X}$ needed to construct such a set $U$ with probability greater than $1-\delta$ has to be at least

\begin{equation}\label{eq:nwsbeta}
\size{D_X} > \beta(\rho/4)\left(\ln{\beta(\rho/8)} + \ln{(1/\delta)}\right),
\end{equation}
\noindent for $\beta(\rho) = \frac{\vol{(\mathcal{\X})}}{\cos^m{(\arcsin{(\frac{\rho}{2\mu})})}\vol{(\B^m_{\rho})}}$, and where $\mu$ is the condition number, $\text{dim}(\X) = m$, and $\rho$ the radius $0< \rho < \mu/2$ of the $m$-balls covering $\X$. 

Hanneke \newcite{Hanneke} also showed that the minimum sample size needed for \eqref{errorprobfunction} to hold with probability at least $1 - \delta $ is given by

\begin{equation}\label{eq:hanneke2}
\size{D_X} \geq c\left(\frac{d + \log{(1/\delta)}}{\epsilon^2}\right)
\end{equation}

\noindent for some constant $c$ and $0 < \epsilon, \delta \leq 1/2$. This bound is tight, but we are only concerned with the lower bound. 

Assume $\size{D_X}$ satisfies \eqref{nwsbeta}. Then there is an assignment of $\beta(\rho)$ that allows \eqref{hanneke2} to hold. 
Solving \eqref{hanneke2} for $\beta(\rho)$ with the ansatz:

\begin{align}
\epsilon &= \sqrt{\frac{c}{\beta(\rho/4)}}, \\
          d &= c'\log{\beta(\rho/8)},
\end{align}

\noindent where $c' = \ln{(2)}$, yields: 

\begin{align}
\epsilon &= \sqrt{c}\left(\frac{\cos^m{(\arcsin{(\frac{\rho}{8\mu}))}}\vol{(\B^m_{\rho/4})}}{\vol{(\mathcal{\X})}}\right)^{1/2}, \\
          d &= c'\log{\left(\frac{\vol{(\mathcal{\X})}}{\cos^m{(\arcsin{(\frac{\rho}{16\mu}))}}\vol{(\B^m_{\rho/8})} }\right)}.
\end{align}

 We refine this result, and obtain the same as \eqsref{lemmahomologyepsilon}{lemmahomologyvc} by noting that the volume $\vol{(\B^m_{\rho})}$ is given in the $m$-dimensional ambient space, which is Euclidean and equipped with the canonical Euclidean metric, and that $\Gamma(\frac{x}{2} + 1) = (x/2)!$: 

\begin{align}
\epsilon &= \sqrt{c}\left(\frac{\cos^m{(\arcsin{(\frac{\rho}{8\mu})})}\pi^{m/2}\left(\frac{\rho}{4}\right)^m}{(\frac{m}{2})!\vol{(\mathcal{\X})}}\right)^{1/2}, \\
          d &= c'\log{\left(\frac{\vol{(\mathcal{\X})}(\frac{m}{2})!}{\cos^m{(\arcsin{(\frac{\rho}{16\mu}))}}\pi^{m/2}\left(\frac{\rho}{8}\right)^m }\right)}.
\end{align}

Letting $\Lambda_\rho = \left(\frac{\rho \sqrt{\pi}}{4}\right)\left(1 - \frac{\rho^2}{64\mu^2} \right)^{1/2}$, we obtain

\begin{align}
\epsilon &\in \BigOmega {\left(\frac{\Lambda_\rho^m}{(\frac{m}{2})!\vol{(\X)}}\right)^{1/2}} \\
d &\in  \BigOmega{\log{\left(2^m\left(\frac{64 - \rho^2/\mu^2}{256 - \rho^2/\mu^2}\right)^{m/2}\frac{(\frac{m}{2})!\vol{(\X)}}{\Lambda_\rho^m}\right)}}
\end{align}

\noindent as desired. It follows that, when the homology of $\cup_{\langle x_i, y_i \rangle \in D} \B_\rho(x_i)$ equals the homology of $\X$, $D$ is $(\epsilon, \delta, W)$-learnable, if $\epsilon$ and $d$ satisfy \eqsref{lemmahomologyepsilon}{lemmahomologyvc}.

\paragraph{"Only if" direction:}

We show that a dataset that is $(\epsilon, \delta, W)$-representative of $\P$, with its size determined by \eqsref{lemmahomologyepsilon}{lemmahomologyvc} and \eqref{hanneke2} has the same homology as $\X$. 

Fix an arbitrary confidence $0 < \delta \leq 1/2$. 
Assume there exists an assignment of $\epsilon$ and $d$ that satisfies \eqref{hanneke2}, such that $D$ is $(\epsilon, \delta, W)$-representative of $\P$ under these conditions. 
Then

\begin{align}
\size{D} &> c\left( \frac{c'\log{\left(2^m\left(\frac{64 - \rho^2/\mu^2}{256 - \rho^2/\mu^2}\right)^{m/2}\frac{(\frac{m}{2})!\vol{(\X)}}{\Lambda_\rho^m}\right)} + \log{\frac{1}{\delta}}}{c_\epsilon\left(\frac{\Lambda_\rho^m}{(\frac{m}{2})!\vol{(\X)}}\right)} \right) \label{eq:datasetsize}
\end{align}

\noindent for some positive, nonzero constants $c \geq \ln{(2)}$, $\log{c'} = c_d \geq 1/8$, and ${c_\epsilon} \geq 1$ encoding the respective constants for \eqsref{lemmahomologyepsilon}{lemmahomologyvc}.

Niyogi et al. \newcite{NiyogiSmale} also showed that the confidence $\delta$ is lower-bounded by the $\rho/4$-packing number $z$, times the minimum probability $1- \alpha$ that the intersection of the cover with a nontrivial subset $F \subset D_X$, $\size{F} < \size{D_X}$ is empty, that is:

\begin{equation}\label{eq:subseteq}
z(1 - \alpha)^{\size{F}} \leq ze^{-\size{F}\alpha} \leq \delta
\end{equation}

\noindent for $\alpha = \mmin{x_i \in F} \vol{(\B_{\rho/4}(x_i)\cap \X)}/\vol{(\X)}$. Crucially, if $\size{F}$ is larger than a carefully-chosen lower bound, it follows that the homology of a cover of $F$ equals the homology of $\X$. 
From the same work we know that $\alpha$ and $z$ are invariants of the manifold, and are unrelated to the dataset itself. In particular,

\begin{align}
\frac{1}{\alpha} &\leq \frac{\vol{(\X)}}{\cos^m{(\arcsin{(\frac{\rho}{8\mu})})} \vol{(\B_{\rho/4})}}, \label{eq:alpha}\\
z &\leq \frac{\vol{(\X)}}{\cos^m{(\arcsin{(\frac{\rho}{16\mu})})} \vol{(\B_{\rho/8})}}. \label{eq:zeta}
\end{align}

Solving for $\size{F}$ in \eqref{subseteq}, we obtain:

\begin{equation}
\size{F} \geq \frac{1}{\alpha}\ln{\frac{z}{\delta}}, 
\end{equation}

\noindent which is precisely the statement from \eqref{nwsbeta}. Plugging in \eqref{datasetsize} in \eqref{subseteq}, along with the fact that $\size{D_X} \geq \size{F}$, it follows that:

\begin{equation}
\frac{1}{\alpha}\ln{\frac{z}{\delta}} \leq  c\left( \frac{\log{\left(c_d2^m\left(\frac{64 - \rho^2/\mu^2}{256 - \rho^2/\mu^2}\right)^{m/2}\frac{(\frac{m}{2})!\vol{(\X)}}{\Lambda_\rho^m}\right)} + \log{\frac{1}{\delta}}}{\left(c_\epsilon\frac{\Lambda_\rho^m}{(\frac{m}{2})!\vol{(\X)}}\right)} \right)
\end{equation}
\begin{equation}
\leq  \left(\frac{c}{\ln{(2)}}\right)\left(c_\epsilon\frac{(\frac{m}{2})!\vol{(\X)}}{\Lambda_\rho^m}\right) \ln{\left(c_d2^m\left(\frac{64 - \rho^2/\mu^2}{256 - \rho^2/\mu^2}\right)^{m/2}\frac{(\frac{m}{2})!\vol{(\X)}}{\delta\Lambda_\rho^m}\right)}.
\end{equation}

Our work is limited to show variable-per-variable bounds. 
Note, however, that from \eqsref{alpha}{zeta} we do not obtain any information with respect to the relationship of these bounds and the ones from \eqsref{lemmahomologyepsilon}{lemmahomologyvc}. Thus the best we may expect is that the bounds from \eqsref{alpha}{zeta} are at least a lower bound for \eqsref{lemmahomologyepsilon}{lemmahomologyvc}:

\begin{align}
z &\leq c_d2^m\left(\frac{64 - \rho^2/\mu^2}{256 - \rho^2/\mu^2}\right)^{m/2}\frac{(\frac{m}{2})!\vol{(\X)}}{\Lambda_\rho^m} \\
\frac{\vol{(\X)}}{\cos^m{(\arcsin{(\frac{\rho}{16\mu})})} \vol{(\B_{\rho/8})}} &\leq c_d2^m\left(\frac{64 - \rho^2/\mu^2}{256 - \rho^2/\mu^2}\right)^{m/2}\frac{(\frac{m}{2})!\vol{(\X)}}{\Lambda_\rho^m} \\
\frac{\vol{(\X)}}{\cos^m{(\arcsin{(\frac{\rho}{16\mu})})} \vol{(\B_{\rho/8})}} &\leq c_d2^m\left(\frac{64 - \rho^2/\mu^2}{256 - \rho^2/\mu^2}\right)^{m/2}\frac{(\frac{m}{2})!\vol{(\X)}}{\left(\frac{\rho \sqrt{\pi}}{4}\right)^m\left(1 - \frac{\rho^2}{64\mu^2} \right)^{m/2}}.
\end{align}

We drop the dependence on the volume of the manifold, and recall that $\vol{(\B_{\rho/8})}$ is defined over the tangent space, which is Euclidean.\footnote{ This last assumption is not necessary, given that the definition of the lower bound of $d$ in \eqref{lemmahomologyvc} implies a constant $c_d$ which can be made arbitrarily large with respect to the volume form of $\B_{\rho/8}$ on a chosen (oriented) manifold.}
Using this, along with the fact that $\cos{(\arcsin{(x)})} = \sqrt{1 - x^2}$, we obtain:

\begin{align}
\frac{1}{\cos{(\arcsin{(\frac{\rho}{16\mu})})}} &\leq c_d\left(\frac{64 - \rho^2/\mu^2}{256 - \rho^2/\mu^2}\right)^{1/2}\left(1 - \frac{\rho^2}{64\mu^2} \right)^{-1/2}, \\
1 &\leq 8(c_d), \\
\end{align}

\noindent which is true by the definition of $c_d$.

Similarly, for $\alpha$, we obtain another inequality:

\begin{align}
\frac{1}{\alpha} &\leq c_\epsilon\left(\frac{(\frac{m}{2})!\vol{(\X)}}{\Lambda_\rho^m}\right) \\
\frac{\vol{(\X)}}{\cos^m{(\arcsin{(\frac{\rho}{8\mu})})} \vol{(\B_{\rho/4})}} &\leq c_\epsilon\left(\frac{(\frac{m}{2})!\vol{(\X)}}{\left(\frac{\rho \sqrt{\pi}}{4}\right)^m\left(1 - \frac{\rho^2}{64\mu^2} \right)^{m/2}}\right) \\
\frac{1}{\cos^m{\left(\arcsin{(\frac{\rho}{8\mu})}\right)}} &\leq c_\epsilon\left(1 - \frac{\rho^2}{64\mu^2} \right)^{-m/2} \\
1 &\leq c_\epsilon,
\end{align}

which is true by the definition of $c_\epsilon$. 

It follows that if $D$ satisfies \eqref{hanneke2} with \eqsref{lemmahomologyepsilon}{lemmahomologyvc} and $c \geq \ln{(2)}$, the homology of $\cup_{\langle x_i, y_i \rangle \in D} \B_\rho(x_i)$ equals the homology of $\X$. 

\paragraph{Conclusion:}

From both parts it follows that $D$ is $(\epsilon, \delta, W)$-learnable if and only if the homology of a cover for $\X$ with elements of $D$, $\cup_{\langle x_i, y_i \rangle \in D} \B_\rho(x_i)$, equals the homology of $\X$, as long as $\epsilon$ and $d$ satisfy \eqsref{lemmahomologyepsilon}{lemmahomologyvc}. This concludes the proof.

\begin{remark}
A consequence of the second part of the proof of \lemref{learnablesamplesize} is that the Big-Omega notation can be substituted with an inequality. 
\end{remark}

\section{Proof of \lemref{tauworks}}\label{app:tauworksproof}

Our proof follows closely the techniques from Niyogi et al. \newcite{NiyogiSmale}, as this problem has a natural solution in terms of their homology-based framework.

Fix a probability $0 <\delta \leq 1/2$, and, for notational convenience, let $E_X = \{x_i \colon \langle x_i, y_i \rangle \in E\}$.

We begin by making two observations: first, by definition, $Z = X \cup \Delta$ lies on some manifold $\Z$, such that $\X \subset \Z$. 
It follows that, by \lemref{homologouslemma}, 
the homology of a cover realized by a subset $F \subset \tilde{E}^{(\kappa)}$ equals the homology of $\X$, provided $F$ is $(\epsilon, \delta, W)$-learnable. 

Second, the $\tau$-function samples elements within a radius $\rho/4$ of its argument; here some $x \in X$. We are guaranteed the existence of at least another $x' \in X$ in this $m$-ball, since $E_X$ is $\rho/2$-dense in $\X$; that is, $\forall x \in E_X$, $\exists x' \in \X.$ $\vert\vert x - x; \vert\vert_{\mathbb{R}^m} \leq \rho/2$. 

Suppose $\Delta = \emptyset$. 
For any $m$-ball centered at $x \in E_X$, from the density of $E_X$ it follows that the minimum size needed to obtain an element that fulfills the conditions of the lemma (w.h.p.) is given by $\size{E_X\cap \B_{\rho/4}(x)} < (1/2)\size{(X\setminus E_X) \cap \B_{\rho/4}(x)}$. Hence 

\begin{equation}\label{eq:minimumdelta}
\size{(E_X \cup \Delta) \cap \B_{\rho/4}(x)} < (1/2)\size{(X \setminus E_X) \cap \B_{\rho/4}(x)}
\end{equation}
\noindent is the minimum set size required when $\Delta \neq \emptyset$. 

We now formalize these observations and establish the number of tries $\kappa$ needed to obtain the centers for the subset $F$.

There are two natural partitions of $\tilde{E}^{(\kappa)}$: one corresponding to the results of every sampling round, as displayed in \eqref{tildeedef}; and another mapping the index of every $i^{\text{th}}$ element of $E_X$ to a set tabulating the result of the $j^{\text{th}}$ sampling with $\tau^{Z}_\rho(\cdot)$. 

Partition $\tilde{E}^{(\kappa)}$ based on the latter strategy, that is, $\hat{E}^{(i)} = \{\tau^{Z, (j)}_\rho(x_i) \colon x_i \in E_X \land j \in [1,\dots,\kappa]\}$, for $\hat{E}^{(i)} \subset \tilde{E}^{(\kappa)}$ and $\cup_{i \in \{1, \dots, \size{E_X}\}} \hat{E}^{(i)} = \tilde{E}^{(\kappa)}$. 
Let $A_i$ be the event that the $i^{\text{th}}$ subset of $\tilde{E}^{(\kappa)}$ under this partition contains solely elements of $E_X$ or $\Delta$, $\hat{E}^{(i)} \cap (X\setminus E_X) = \emptyset$. 

Via the union bound, we get:

\begin{align}\label{eq:fullunionbound}
\text{Pr}\left[\bigcup_{i=1}^{\size{E_X}} [A_i = 1] \right] &\leq \sum_i^{\size{E_X}} \text{Pr}[A_i = 1], \\
&\leq \sum_i^{\size{E_X}} \prod_{j=1}^{\kappa}\left(1 - \frac{\vol({(X \setminus E_X) \cap (\B_{\rho/4}(x_i)\setminus\{x_i\})})}{\vol{(Z \cap (\B_{\rho/4}(x_i)\setminus\{x_i\}))}}\right) \\
&\leq \sum_i^{\size{E_X}} \left(1 - \frac{\vol({(X \setminus E_X) \cap (\B_{\rho/4}(x_i)\setminus\{x_i\})})}{\vol{(Z \cap (\B_{\rho/4}(x_i)\setminus\{x_i\}))}}\right)^{\kappa} \label{eq:firstequation}, 
\end{align}

\noindent since all elements in $\Q_{x_i, \rho}$ are equally probable. 
Let 
\begin{equation}\label{eq:alphafortau}
\alpha = \mmin{i \in \{1, \dots, \size{E_X}\}}\frac{\vol({(X \setminus E_X) \cap (\B_{\rho/4}(x_i)\setminus\{x_i\})})}{\vol{(Z \cap (\B_{\rho/4}(x_i)\setminus\{x_i\}))}}.
\end{equation}

By relying on the $\rho/2$-density of $E$ on $\X$, we obtain $\alpha > 0$. The subspace of $\Z$ inside the cover realized by $E$ has the same condition number as $\X$, and we may then relax the geometry of this problem. We lower bound \eqref{alphafortau} with a counting argument: 
since all $\tau_{\rho}^Z{x_i}$ are obtained i.i.d. from $Z$ from a uniform probability distribution, assume that every draw for every $\B_{\rho/4}(x_i)$ is identical, at least with respect to $Z$. This is not a strong assumption since $\langle x_i, y_i \rangle \in E$ were also sampled i.i.d. 
We drop the dependence on the volume of $\B_{\rho/4}(x_i)$ by noting that we do not sample elements strictly from $X$. 
From \eqref{minimumdelta} and $\size{Z} = \size{X} + \size{\Delta}$,
we obtain: 

\begin{align}
\alpha &\geq \frac{\size{\{\tilde{x_i} \colon \tilde{x_i} \in (X \setminus E_X) \cap (\B_{\rho/4}(x_i)\setminus\{x_i\}) \}}}{\size{\{\tilde{x_i} \colon \tilde{x_i} \in Z \cap (\B_{\rho/4}(x_i)\setminus\{x_i\}\})}} \\
\alpha &\geq \frac{\size{X \setminus E_X}}{\size{Z}} \\
\alpha &\geq \frac{\size{X \setminus E_X}}{\size{X} + \size{\Delta}}. \label{eq:zminusdeltabound} 
\end{align}

Following \newcite{NiyogiSmale} we simplify \eqref{zminusdeltabound} to

\begin{equation}\label{eq:bigenchilada}
\text{Pr}\left[\bigcup_{i=1}^{\size{E}} [A_i = 1] \right] \leq \size{E}\left(1 - \alpha\right)^{\kappa} \leq \size{E}e^{-\kappa \alpha} \leq \delta,
\end{equation}
 
\noindent where we used the fact that $(1 - \beta) \leq e^\beta$ for $\beta \geq 0$. Choose $\kappa$ to be 

\begin{align}
\kappa &\geq \frac{1}{\alpha}\left( \ln{\size{E}} + \ln{(1/\delta)}\right) \\
\kappa &\geq \left(\frac{\size{X} + \size{\Delta}}{\size{X} - \size{E}}\right) \ln{\left(\frac{\size{E}}{\delta}\right)} \label{eq:kappaeq}.
\end{align}

Hence $\tilde{E}^{(\kappa)}$ has a subset $F \subset \tilde{E}^{(\kappa)}$ that realizes a cover of $\X$, and such that $\size{F} \geq \size{E}$. 

We may also bound $\size{\Delta}$ from \eqref{minimumdelta}, 

\begin{equation}\label{eq:deltaeq}
2\size{\Delta} < \size{X} - 3\size{E_X}. 
\end{equation} 

It follows that a cover of $F$ has the same homology as $E$; by \lemref{homologouslemma} we conclude that $\tilde{E}^{(\kappa)}$ is $(\epsilon, \delta, W)$-learnable when \eqref{kappaeq} and \eqref{deltaeq} hold.

\begin{remark}
It is possible to obtain a tighter bound on \eqref{alphafortau} (and thus on \eqref{kappaeq}) by employing a geometric argument. We leave this question open for further research. 
\end{remark}

\section{Proof of \thmref{timecomplexity}}\label{app:timecomplexityproof}

Let $k$ be an iteration of Agora, and let $T_f(\size{D^{(k)}})$ be the upper bound number of steps required to train each of the $\Theta^{(k)}$ Timaeus models $f(\cdot; \cdot)$ on $D^{(k)}$. 
Then the cost, per iteration, of \lineref{startfor} to \lineref{endfor} is

\begin{equation}\label{eq:startforendforcost}
\BigO{\size{\Theta^{(k)}}\left(T_f(\size{D^{(k)}}) + \bar{f}\size{E}\right)}.
\end{equation}

\linestworef{sortstatement}{checkstatement} involve $\size{M^{(k)}}$ calls to $\tau^{Z}_\rho(\cdot)$, and $\size{M^{(k)}}$ calls to Socrates. 
Both operations can be executed at the same time, and, since Timaeus has perfect memory, we know that the bound corresponding to this step is tight. 
Moreover, \lineref{sortqstatement} is a simple sorting statement, which we assume to be comparison-based, while \linestworef{selectstatement}{removestatement} can be executed with a na\"ive implementation that compares every element and employs a counting argument, in time $\operatorname{O}(\size{Q^{(k)}}\size{\theta}^2)$, for $\theta \in \Theta^{(k)}$.

This yields:

\begin{equation}
\BigO{\size{M^{(k)}}\bar{S} + \size{Q^{(k)}}\size{\theta}^2 + \size{Q^{(k)}}\log{\size{Q^{(k)}}}}.
\end{equation}

Adding both equations together, and using the fact that $\size{Q^{(k)}} = \size{\Theta^{(k)}}$, gives a total time per iteration of

\begin{equation}\label{eq:periterationtime}
\text{Cost}(k) = \BigO
{\size{\Theta^{(k)}}\left(T_f(\size{D^{(k)}}) + \bar{f}\size{E} + \size{\theta}^2 + \log{\size{\Theta^{(k)}}}\right) +
 \size{M^{(k)}}\bar{S} }.
\end{equation}

Since in the worst case Agora prunes out exactly one member of $\Theta$ at each iteration, \eqref{periterationtime} can be expressed in terms of $\size{\Theta}$. 
By using \lemref{periterationperf} we obtain:

\begin{multline}\label{eq:bigenchilada2}
\sum_{k=1}^{\size{\Theta}} \text{Cost}(k) = \sum_{k=1}^{\size{\Theta}}\size{\Theta^{(k)}}\left(\bar{f}\size{E} + T_f(\size{D^{(k)}})\right) +  \\
 \BigO{\size{\Theta}^2\left(\log{\size{\Theta}} + 
 \size{\theta}^2 \right) +
 \bar{S}\size{\Theta}\size{E}\left(1 - \frac{1}{2^{\size{\Theta}}}\right)}.
\end{multline}

Note that, for $k \geq 2$,
\begin{equation}
\sum_{k=1}^{\size{\Theta}}\size{\Theta^{(k)}}T_f(\size{D^{(k)}}) = \size{\Theta}T_f(\size{D}) +\sum_{k=2}^{\size{\Theta}}(\size{\Theta} - (k-1))T_f\left(\size{D} + \frac{\size{E}}{2^{k-1}}\right),
\end{equation}

\noindent and so
\begin{align}
\sum_{k=1}^{\size{\Theta}}\left(\bar{f}\size{E} + T_f(\size{D^{(k)}})\right) \in \BigO{\size{\Theta}^2\left(\bar{f}\size{E} + T_f\left(\size{D} + \frac{\size{E}}{2^{\size{\Theta}-1}}\right)\right)}.
\end{align}

It then follows that \eqref{bigenchilada2} simplifies to 

\begin{equation}
\sum_{k=1}^{\size{\Theta}} \text{Cost}(k) \in \BigO{\size{\Theta}^2\left(\log{\size{\Theta}} + 
\size{\theta}^2 + T_f\left(\size{D} + \frac{\size{E}}{2^{\size{\Theta}-1}}\right)\right) + 
\right.   
\left.
\bar{S}\size{\Theta}\size{E}\left(1 - \frac{1}{2^{\size{\Theta}}}\right)},
\end{equation}

\noindent as desired.

\section{Proof of \corref{sgdcor}}\label{app:nnruntimecomplexityproof}

Let $\pi_i = \{\eta_i, \size{B}_i\}$, and note that every combination of elements of $\pi_i$ guarantees an $\epsilon_i$-accurate solution (w.p.1) in polynomial time under the assumptions given. The unknown $\epsilon_i$ is fully characterized by $\pi_i$ and, by definition, the rest of the hyperparameters from $\theta_i$. 

It follows that constructiong a hyperparameter set $\theta_i' = \theta_i \cup \pi_i$ allows Agora to search over $\pi_i$, and select the appropriate values based on $\epsilon_i$, hence maintaining the invariants from \linestworef{maxstatement}{removestatement}, which, by \lemref{periterationperf} and \thmref{timecomplexity}, ensure correctness. 

The bound for \eqref{nnruntimecomplexity} is obtained by a substitution in \eqref{runtimecomplexity} of $T_f(D)$. Here $T_f(D)$ is bounded by using the fact that \train$(f, B_i, \theta_i')$ with shuffling-type SGD runs in $\operatorname{O}(\size{B}_i\bar{f}/\epsilon_i^{2/3})$ steps \cite{NguyenWainwrightJordan}; and that $\forall \pi_i, \theta_i, \epsilon_j$, \train$(f, B_i, \theta_i)$ induces a surjection $\pi_i \cup \theta_i \mapsto \epsilon_j$, 

\section{Experimental Results}\label{app:experiments}

In this section we compare the accuracy and runtime between Agora and an enumeration (i.e., picking the best model based on the evaluation set $\forall \theta \in \Theta$) for a benchmark dataset. 
We use the Breast Cancer Wisconsin (Diagnostic) Dataset \cite{UWDataset} from UCI's Machine Learning Repository \cite{Dua}. 
It is a binary classification task comprised of $569$ samples. The learner is given a vector in $\mathbb{R}^{30}$ corresponding to the characteristics of a cell nucleus (e.g. radius, smoothness, and symmetry) and must determine whether it is malignant or benign. 

For both experiments, our hypothesis space is a set of decision trees, with their depth as the search space: $\Theta = \{\{1\}, \{2\}, \{3\}, \{4\}, \{5\}\}$.\footnote{Formally, this means that Timaeus has variable VC dimension. This does not affect the results, since the hypothesis space is fixed.} 
We split the dataset in two, reserving $25\%$ for the test set. This task is known to be easy to solve \cite{Olvietal}; therefore, we adapt it to our setting by removing all but the first $5$ examples from the training set. The next $30$ are used as our development set, and we discard the rest. 
We use a random seed of $0$ and the Gini impurity as the objective function. 

For Agora, we train a depth-$5$ decision tree in Scikit-learn \cite{scikit-learn}. 
This tree trains in $4$ milliseconds on a 6-core Intel i7 processor running at 2.6 GHz. It obtains an accuracy of $90\%$ on the test set, and we use it as the Socrates model in our evaluation of Agora. 
The $\tau$-function is chosen to inject $(0,1)$-Gaussian noise in the input at random coordinates with a probability of $2/5$. The Gaussian noise both simulates the set $\Delta$ from the $\tau$-function's codomain, and provides a reasonable $\rho$ since all features from the points in the dataset are normalized. 

See \tabref{results} for a comparison between both approaches. Agora's final model obtains an accuracy of $93.7\%$ on the test set. In contrast, the enumeration over $\Theta$ obtains an accuracy of $72.0\%$. However, the enumeration takes $7$ milliseconds to run, while Agora takes $39.8$ milliseconds. These results are well within the estimated runtime bound of $\operatorname{O}(\size{\Theta}^2)$ from \secref{discussioncomp} and the performance bound from \thmref{maintheorem}. 

\begin{table}
\centering
\begin{tabular}{ |c|c|c|c| } \hline
 Algorithm   & Accuracy (dev) & Accuracy (test) & Time elapsed (ms) \\ \hline\hline
 Enumeration & $75.0$          & $72.0$         & $\mathbf{7.0}$  \\
 Agora       & $\mathbf{95.0}$          & $\mathbf{93.7}$         & $39.8$ \\ \hline
\end{tabular}
\caption{Results for a run of Agora on the Breast Cancer Wisconsin (Diagnostic) Dataset, when compared to an enumeration over $\Theta$, and using decision trees as the Socrates and Timaeus models. Agora's resulting Timaeus model outperforms an enumeration in both splits of the dataset, in addition to outperforming Socrates by $3.7\%$. Note that the runtime is well within the estimated theoretical bound of $\operatorname{O}(\size{\Theta}^2)$ from \secref{discussioncomp}.}
\label{tab:results}
\end{table}

\end{document}